\documentclass[letterpaper, 10 pt, journal, twoside]{IEEEtran}  
\IEEEoverridecommandlockouts  

\usepackage{amsthm}
\usepackage[dvipsnames]{xcolor} 
\usepackage{tikz} 
\usetikzlibrary{shapes,snakes} 
\usetikzlibrary{tikzmark} 
\usetikzlibrary{calc, arrows.meta, positioning} 
\usepackage{multirow} 
\usepackage{algorithm} 
\usepackage{algpseudocode} 
\usepackage{array} 
\usepackage{booktabs} 
\usepackage{amssymb}  
\usepackage{mathtools, empheq,eqparbox}  
\usepackage{subcaption} 
\usetikzlibrary{shapes.multipart} 
\usepackage{makecell,balance,multirow}
\usepackage{float} 
\usepackage[most]{tcolorbox} 
\usepackage{svg} 
\usepackage{pdfpages} 
\usepackage[pdftex, bookmarksopen=true,plainpages = false, colorlinks=true, linkcolor=black, citecolor = black, urlcolor = black,filecolor=black , pagebackref=false,hypertexnames=false, plainpages=false, pdfpagelabels,unicode=true,hyperfootnotes=false ]{hyperref}
\usepackage[sort,compress]{cite} 
\usepackage{subcaption}
\usepackage{MnSymbol} 

\DeclareCaptionLabelSeparator{periodspace}{.\quad}
\usepackage[binary-units=true]{siunitx}
\usepackage{mdframed}

\usepackage{pifont}
\newcommand*\colourcheck[1]{%
  \expandafter\newcommand\csname #1check\endcsname{\textcolor{#1}{\ding{52}}}%
}
\colourcheck{blue}
\colourcheck{green}
\colourcheck{red}

\newcommand*\colourcross[1]{%
  \expandafter\newcommand\csname #1cross\endcsname{\textcolor{#1}{\ding{55}}}%
}
\colourcross{blue}
\colourcross{green}
\colourcross{red}

\usepackage[dvipsnames]{xcolor}

\title{\LARGE \bf
Robust MADER: Decentralized Multiagent Trajectory Planner Robust to Communication Delay in Dynamic Environments
}
 
\newcommand{\trajOpt}{traj\textsubscript{opt}}
\newcommand{\trajComm}{traj\textsubscript{comm}}

\newcommand{\trajAOpt}{traj\textsubscript{A\textsubscript{opt}}}
\newcommand{\trajAComm}{traj\textsubscript{A\textsubscript{comm}}}

\newcommand{\trajBOpt}{traj\textsubscript{B\textsubscript{opt}}}
\newcommand{\trajBComm}{traj\textsubscript{B}\textsubscript{comm}}

\newcommand{\trajJOpt}{traj\textsubscript{J\textsubscript{opt}}}
\newcommand{\trajJComm}{traj\textsubscript{J}\textsubscript{comm}}

\newcommand{\delayParameter}{\ensuremath{\delta_\text{DC}}}
\newcommand{\delayActual}{\ensuremath{\delta_\text{actual}}}
\newcommand{\delayActualMax}{\ensuremath{\delta_\text{max}}}
\newcommand{\delayIntroduced}{\ensuremath{\delta_\text{introd}}}

\newcommand{\NeccessaryCond}{\ensuremath{\delayActualMax{}\le\delayParameter{}}}

\newcommand{\NoRed}{\textbf{\textcolor{red}{No}}}
\newcommand{\YesGreen}{\textbf{\textcolor{ForestGreen}{Yes}}}

\definecolor{opt_color}{RGB}{203,255,182}
\definecolor{check_color}{RGB}{195,218,255}
\definecolor{recheck_color}{RGB}{255,176,176}
\definecolor{delaycheck_color}{RGB}{255,228,181}

\newcommand\mybox[2][]{\tikz[overlay]\node[fill=blue!20,inner sep=0.6pt, anchor=text, rectangle, rounded corners=0mm,#1] {#2};\phantom{#2}}

\newcommand{\OptimizationStep}{\mybox[fill=opt_color]{Optimization}}
\newcommand{\CheckStep}{\mybox[fill=check_color]{Check}}

\newcommand{\DelayCheckStep}{\mybox[fill=delaycheck_color]{Delay Check}}

\newcommand{\OStep}{\mybox[fill=opt_color]{O}}
\newcommand{\CStep}{\mybox[fill=check_color]{C}}

\newcommand{\DCStep}{\mybox[fill=delaycheck_color]{DC}} 

\newcommand{\DeltaO}{\mybox[fill=opt_color]{$\delta_{\text{O}}$}}
\newcommand{\DeltaC}{\mybox[fill=check_color]{$\delta_{\text{C}}$}}

\newcommand{\DeltaDC}{\mybox[fill=delaycheck_color]{$\delta_{\text{DC}}$}}

\newcommand{\DeltaOA}{\mybox[fill=opt_color]{$\delta_{\text{O}_{A}}$}}
\newcommand{\DeltaCA}{\mybox[fill=check_color]{$\delta_{\text{C}_{A}}$}}

\newcommand{\DeltaDCA}{\mybox[fill=delaycheck_color]{$\delta_{\text{DC}_{A}}$}}

\newcommand{\DeltaOB}{\mybox[fill=opt_color]{$\delta_{\text{O}_{B}}$}}
\newcommand{\DeltaCB}{\mybox[fill=check_color]{$\delta_{\text{C}_{B}}$}}

\newcommand{\DeltaDCB}{\mybox[fill=delaycheck_color]{$\delta_{\text{DC}_{B}}$}}

\newcommand{\OStepA}{\mybox[fill=opt_color]{O\textsubscript{A}}}
\newcommand{\CStepA}{\mybox[fill=check_color]{C\textsubscript{A}}}

\newcommand{\DCStepA}{\mybox[fill=delaycheck_color]{DC\textsubscript{A}}} 

\newcommand{\OStepB}{\mybox[fill=opt_color]{O\textsubscript{B}}}
\newcommand{\CStepB}{\mybox[fill=check_color]{C\textsubscript{B}}}

\newcommand{\DCStepB}{\mybox[fill=delaycheck_color]{DC\textsubscript{B}}}

\newcommand{\tApub}{t\textsubscript{A}\textsuperscript{pub}}
\newcommand{\tBrec}{t\textsubscript{B}\textsuperscript{rec}}

\definecolor{agent1_color}{RGB}{234,153,153}
\definecolor{agent2_color}{RGB}{151,104,175}
\definecolor{agent3_color}{RGB}{249,203,156}
\definecolor{agent4_color}{RGB}{194,115,156}
\definecolor{agent5_color}{RGB}{159,197,232}
\definecolor{agent6_color}{RGB}{117,186,117}
\definecolor{obstacle_color}{RGB}{160,117,109}

\newcommand{\MADER}{\textbf{MADER}}
\newcommand{\WOCHECKRMADER}{\textbf{RMADER w/o Check}}
\newcommand{\RMADER}{\textbf{RMADER}}
\newcommand{\EGOswarm}{\textbf{EGO-Swarm}}
\newcommand{\EDGteam}{\textbf{EDG-Team}}

\usepackage{scalerel}
\usepackage{tikz}
\usetikzlibrary{svg.path}

\definecolor{orcidlogocol}{HTML}{A6CE39}
\tikzset{
    orcidlogo/.pic={
        \fill[orcidlogocol] svg{M256,128c0,70.7-57.3,128-128,128C57.3,256,0,198.7,0,128C0,57.3,57.3,0,128,0C198.7,0,256,57.3,256,128z};
        \fill[white] svg{M86.3,186.2H70.9V79.1h15.4v48.4V186.2z}
        svg{M108.9,79.1h41.6c39.6,0,57,28.3,57,53.6c0,27.5-21.5,53.6-56.8,53.6h-41.8V79.1z M124.3,172.4h24.5c34.9,0,42.9-26.5,42.9-39.7c0-21.5-13.7-39.7-43.7-39.7h-23.7V172.4z}
        svg{M88.7,56.8c0,5.5-4.5,10.1-10.1,10.1c-5.6,0-10.1-4.6-10.1-10.1c0-5.6,4.5-10.1,10.1-10.1C84.2,46.7,88.7,51.3,88.7,56.8z};
    }
}

\newcommand\orcidicon[1]{\href{https://orcid.org/#1}{\mbox{\scalerel*{
                \begin{tikzpicture}[yscale=-1,transform shape]
                \pic{orcidlogo};
                \end{tikzpicture}
            }{|}}}}

\begin{document}
\author{Kota Kondo\textsuperscript{\orcidicon{0000-0002-2356-1359}}, Reinaldo Figueroa, Juan Rached, Jesus Tordesillas\textsuperscript{\orcidicon{0000-0001-6848-4070}}, Parker C. Lusk\textsuperscript{\orcidicon{0000-0002-9755-1276}}, and Jonathan P. How\textsuperscript{\orcidicon{0000-0001-8576-1930}}
\thanks{Manuscript received: June 26, 2023; Revised: October 8, 2023; Accepted: November 26, 2023}%
\thanks{This paper was recommended for publication by Editor M. Ani Hsieh upon evaluation of the Associate Editor and Reviewers’ comments.}
\thanks{The authors are with the Aerospace Controls Laboratory, MIT, 77 Massachusetts Ave, Cambridge, MA, USA
{\tt \{kkondo, reyfp, jrached, jtorde, plusk, jhow\}@mit.edu}}
\thanks{This work is supported by Boeing Research \& Technology and the Air Force Office of Scientific Research MURI FA9550-19-1-0386.}
\thanks{Digital Object Identifier (DOI): see top of this page.}}

\markboth{IEEE ROBOTICS AND AUTOMATION LETTERS. PREPRINT VERSION. ACCEPTED NOVEMBER, 2023}
{Kondo \MakeLowercase{\textit{et al.}}: ROBUST MADER: DECENTRALIZED MULTIAGENT TRAJECTORY PLANNER ROBUST TO COMMUNICATION DELAY} 

\maketitle

\begin{abstract}
Communication delays can be catastrophic for multiagent systems. However, most existing state-of-the-art multiagent trajectory planners assume perfect communication and therefore lack a strategy to rectify this issue in real-world environments.
To address this challenge, we propose Robust MADER (RMADER), a decentralized, asynchronous multiagent trajectory planner robust to communication delay. 
RMADER ensures safety by introducing (1) a Delay Check step, (2) a two-step trajectory publication scheme, and (3) a novel trajectory-storing-and-checking approach.
Our primary contributions include: proving recursive feasibility\footnote{We define recursive feasibility as the collision-free guarantee of generating safe trajectories throughout a series of trajectory re-planning.} for collision-free trajectory generation in asynchronous decentralized trajectory-sharing, simulation benchmark studies, and hardware experiments with different network topologies and dynamic obstacles.
We show that RMADER outperforms existing approaches by achieving a 100\% success rate of collision-free trajectory generation, whereas the next best asynchronous decentralized method only achieves 83\% success.
\end{abstract}





\section*{Supplementary Material}
\noindent\textbf{Video}: \href{https://youtu.be/i1d8di2Nrbs}{https://youtu.be/i1d8di2Nrbs} \\
\textbf{Code}: \href{https://github.com/mit-acl/rmader}{https://github.com/mit-acl/rmader}

\section{INTRODUCTION}\label{sec:intro}
The field of multiagent UAV trajectory planning has received substantial attention due to its extensive array of applications~\cite{peng2022obstacle, ryou_cooperative_2022, vinod_safe_2022, kuwata_cooperative_2011, van2017distributed, firoozi2020distributed, luis2020online, gao2022meeting, sebetghadam2022distributed, batra_decentralized_2022, wang2022robust}.
For a trajectory planner to be effectively deployed in real-world scenarios, resilience to both communication delays and dynamic environments is crucial. Yet, the simultaneous attainment of robustness against these two factors remains unexplored in existing studies.

Multiagent trajectory planners can be centralized~\cite{park_efficient_2020, sharon_conflict-based_2015, robinson_efficient_2018} or decentralized~\cite{tordesillas_mader_2022, zhou_ego-swarm_2020, lusk_distributed_2020}.
While centralized methods involve a single machine that plans every agent's trajectory, decentralized approaches enable each agent to plan its trajectory independently.
The latter showcase enhanced scalability and offer robustness against the potential failure of the centralized machine.

It is also worth noting that there are two layers of decentralization---decentralized planning and decentralized communication architecture.
Even if the planning algorithm is decentralized, agents may still require a centralized communication architecture.

While some decentralized methodologies~\cite{wang2017safety,zhou2017fast, fan2020distributed, semnani2020multiagent, chen2017decentralized, chen2021guaranteed} employ state-sensing techniques and forego the need for inter-robot communication, other strategies~\cite{chen_decoupled_2015, liu_towards_2018, lusk_distributed_2020, park_online_2022, toumieh_decentralized_2022, hou_enhanced_2022, cap_asynchronous_2013, tordesillas_mader_2022, zhou_ego-swarm_2020, senbaslar_asynchronous_2022} necessitate inter-agent communication. The non-communicating, state-sensing methods, being free from communication-associated challenges, formulate paths based on the past and present state of other agents, which can sometimes lead to less optimal planning.

\begin{figure}[!t]
    \centering
    \includegraphics[width=\columnwidth]{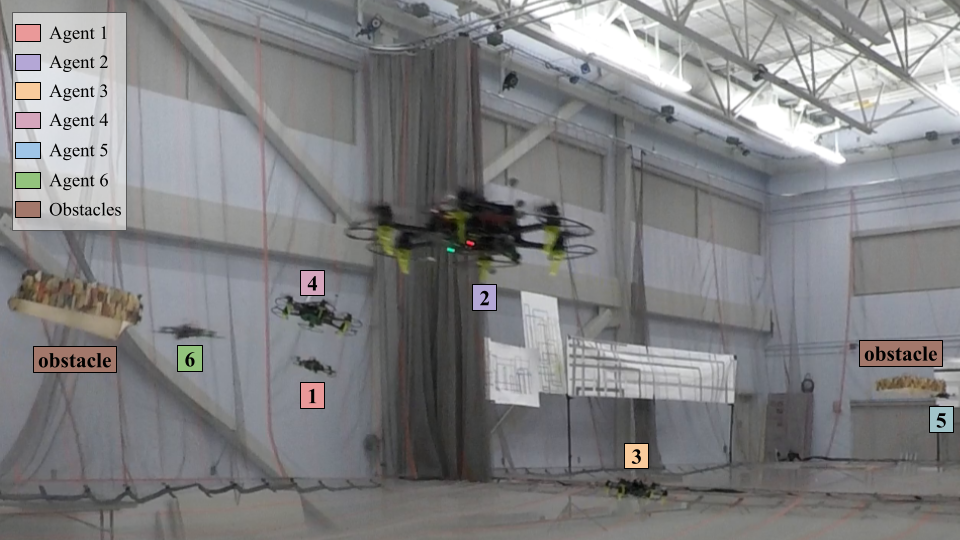}\\
    \caption{
    Hardware demonstration of RMADER with 6 UAVs and 2 dynamic obstacles (UAVs shrouded with brick-pattern foam).
    Agents communicate over a mesh network that exhibits an average delay of 49.8 ms (483 ms max).
    Despite these communication delays, RMADER enables decentralized, asynchronous onboard planning resulting in UAVs flying with a maximum speed of 5.8 m/s.}
    \vspace{-1em}
    \label{fig:mader_hw_6_uavs}
    \vspace{-1em}
\end{figure}

Contrarily, methods requiring communication enable the sharing of future trajectories, improving the efficiency of the overall team's path planning.
While a considerable number of state-sensing, non-communicating methods have been suggested for ground robots~\cite{chen2017decentralized, fan2020distributed, wang2017safety}, their application to UAVs has been limited.
This is mainly due to the increased complexity of 3D state-sensing, prompting researchers to focus primarily on communication-dependent multiagent trajectory planners for UAVs.

However, communication delays can potentially lead to trajectory deconfliction failures~\cite{gielis_critical_2022}.
Given this, the focus of this article is on enhancing the robustness of communication-requisite methods, which inherently exhibit greater efficiency compared to non-communicating strategies.

Multiagent planners can also be classified according to whether or not they are asynchronous.
Asynchronous planning enables each agent to independently trigger the planning step without considering the planning status of other agents.
In contrast to synchronous planners, which require all agents to wait at a so-called synchronization barrier until planning can be globally triggered, asynchronous methods tend to be more scalable.
They are, however, also more susceptible to communication delays since agents plan and execute trajectories independently.
See Table~\ref{tab:multiagent_category} for a summary of the aforementioned multiagent trajectory planner categorization.
\begin{table}[h]
    \renewcommand{\arraystretch}{1}
    \scriptsize
    \begin{centering}
    \caption{\centering Multiagent Trajectory Planner Category}
    \label{tab:multiagent_category}
    \begin{tabular}{>{\centering\arraybackslash}m{0.05\columnwidth} >{\centering\arraybackslash}m{0.15\columnwidth} || >{\centering\arraybackslash}m{0.3\columnwidth} >{\centering\arraybackslash}m{0.3\columnwidth}}
    & & \tikzmark{a} & \tikzmark{b} \tabularnewline
    & & \textbf{Synchronous} & \textbf{Asynchronous} \tabularnewline
    \hline 
    \hline 
    \tikzmark{c} & \textbf{Centralized} & \cite{park_efficient_2020, sharon_conflict-based_2015, robinson_efficient_2018} & Not Possible \tabularnewline
    \cline{2-4}
    \tikzmark{d} & \textbf{Decentralized} & \cite{chen_decoupled_2015, liu_towards_2018, park_online_2022, toumieh_decentralized_2022} & \cite{cap_asynchronous_2013, tordesillas_mader_2022, zhou_ego-swarm_2020, senbaslar_asynchronous_2022}, and \textbf{RMADER} \tabularnewline
    \end{tabular}
    \par\end{centering}
    \begin{tikzpicture}[overlay, remember picture, shorten >=.5pt, shorten <=.5pt, transform canvas={yshift=.25\baselineskip}]
        \draw[thick, -stealth] ({pic cs:a}) -- ({pic cs:b}) node[midway, fill=white, text opacity=1] {\scriptsize Scalability};
        \draw[thick, -stealth] ({pic cs:c}) -- ({pic cs:d}) node[midway, fill=white, fill opacity=0.6, text opacity=1] {\scriptsize Scalability};
    \end{tikzpicture}
\vspace{-1em}
\end{table}

As summarized in Table~\ref{tab:state_of_the_art_comparison}, many state-of-the-art decentralized trajectory planners for UAVs do not consider communication delays or explicitly state their assumptions. 
\begin{table}[h]
    \renewcommand{\arraystretch}{1.1}
    \scriptsize
    \begin{centering}
    \caption{\centering State-of-the-art Decentralized Multiagent Planners for UAVs}
    \label{tab:state_of_the_art_comparison}
    \resizebox{1.0\columnwidth}{!}{
    \begin{tabular}{>{\centering\arraybackslash}m{0.2\columnwidth} >{\centering\arraybackslash}m{0.08\columnwidth} >{\centering\arraybackslash}m{0.1\columnwidth} >{\centering\arraybackslash}m{0.1\columnwidth} >{\centering\arraybackslash}m{0.1\columnwidth} >{\centering\arraybackslash}m{0.1\columnwidth} >{\centering\arraybackslash}m{0.1\columnwidth}}
    \toprule 
    \textbf{Method} & \rotatebox{55}{\textbf{Async?}} & \rotatebox{55}{\textbf{Robust to comm. delay?}} & \rotatebox{55}{\textbf{Handle comm. drops?}} & \rotatebox{55}{\textbf{\makecell{Need bounds of comm. delay?}}} & \rotatebox{55}{\textbf{Hardware in dynamic env.?}} & \rotatebox{55}{\textbf{Mesh network flights}} \tabularnewline
    \midrule 
    \textbf{SCP}~\cite{chen_decoupled_2015}  & \multirow{3}[1]{*}{\NoRed{}} & \multirow{3}[1]{*}{\NoRed{}} & \multirow{3}[1]{*}{\NoRed{}} & \multirow{3}[1]{*}{\NoRed{}} & \multirow{3}[1]{*}{\NoRed{}} & \multirow{3}[1]{*}{\NoRed{}} \tabularnewline
    \cline{0-0}
    \textbf{decNS}~\cite{liu_towards_2018} &&&& \tabularnewline
    \cline{0-0}
    \textbf{LSC}~\cite{park_online_2022} &&&& \tabularnewline
    \hline 
    \textbf{decMPC}~\cite{toumieh_decentralized_2022} & \NoRed{} & \YesGreen{} & \NoRed{} & \NoRed{} & \NoRed{} & \NoRed{} \tabularnewline
    \hline
    \textbf{EDG-Team}~\cite{hou_enhanced_2022} & \YesGreen{}/\NoRed{}\footnotemark[5] & \NoRed{} & \NoRed{} & \NoRed{} & \NoRed{} & \NoRed{} \tabularnewline
    \hline 
    \textbf{ADPP}~\cite{cap_asynchronous_2013} & \YesGreen{}\footnotemark[6] & \NoRed{} & \NoRed{} & \NoRed{} & \NoRed{} & \NoRed{} \tabularnewline
    \hline
    \textbf{MADER}~\cite{tordesillas_mader_2022} & \YesGreen{} & \NoRed{} & \NoRed{} & \NoRed{} & \NoRed{} & \NoRed{} \tabularnewline
    \hline 
    \textbf{EGO-Swarm}~\cite{zhou_ego-swarm_2020} & \YesGreen{} & \NoRed{} & \NoRed{} & \NoRed{} & \NoRed{} & \NoRed{} \tabularnewline
    \hline 
    \textbf{AsyncBVC}~\cite{senbaslar_asynchronous_2022} & \YesGreen{} & \YesGreen{}  & \YesGreen{} & \YesGreen{} & \NoRed{} & \NoRed{} \tabularnewline
    \hline
    \textbf{RMADER}\ (proposed) & \YesGreen{} & \YesGreen{} & \NoRed{} & \NoRed{} & \YesGreen{} & \YesGreen{} \tabularnewline
    \bottomrule
    \end{tabular}}
    \par\end{centering}
\vspace*{0.5em}
\footnotesize{$^5$ \!\!\! EDG-Team triggers joint optimization in dense environments and switches to a centralized, synchronous planner.} \\
\footnotesize{$^6$  \!\!\!\! Asynchronous but requires priority information for planning.}
\vspace{-1em}
\end{table}

For example, \textbf{SCP}~\cite{chen_decoupled_2015}, \textbf{decNS}~\cite{liu_towards_2018}, and \textbf{LSC}~\cite{park_online_2022} are decentralized and synchronous. 
However, SCP and decNS make implicit assumptions, and LSC explicitly presumes a perfect communication environment.

The algorithm \textbf{decMPC}~\cite{toumieh_decentralized_2022} is decentralized, but it requires synchronicity and communication delays to be within a fixed planning period.
\textbf{EDG-Team}~\cite{hou_enhanced_2022} is a decentralized semi-asynchronous planner, which solves joint optimization as a group. 
EDG-Team cooperatively tackles the  problem but implicitly assumes no communication delays. 
\textbf{ADPP}~\cite{cap_asynchronous_2013} is asynchronous\footnote{As in \cite{tordesillas_mader_2022}, we define asynchronous planning to be when the agent triggers trajectory planning independently without considering the planning status of other agents. 
However, ADPP~\cite{cap_asynchronous_2013} implements a prioritized asynchronous approach, meaning plannings are not fully independently triggered.} and decentralized, but it assumes perfect communication without delay. 
Our previous work \MADER{}~\cite{tordesillas_mader_2022} is asynchronous and decentralized but assumes no communication delays.
\textbf{EGO-Swarm}~\cite{zhou_ego-swarm_2020} also proposes a decentralized, asynchronous planner that requires agents to periodically broadcast a trajectory at a fixed frequency, and each agent immediately performs collision checks upon receiving the message. EGO-Swarm is the first fully decentralized, asynchronous trajectory planner successfully demonstrating hardware experiments, yet it still suffers from collisions due to communication delays, as shown in Section~\ref{sec:sim}. 
\textbf{AsyncBVC}~\cite{senbaslar_asynchronous_2022} proposes an asynchronous decentralized trajectory planner that can guarantee safety even with communication delays and drops, and this algorithm does not require the knowledge of the bound of delays.  
However, its future trajectories are constrained by past separating planes, which can overconstrain the solution space and hence increase conservatism.
Further, it relies on discretization when solving the optimization problem, meaning that safety is only guaranteed on the discretization points. 
In contrast, our approach is able to guarantee safety in a continuous approach by leveraging the MINVO basis~\cite{tordesillas_minvo_2022}.

Additionally, to achieve reliable real-world deployment, which involves not only static obstacles, but also dynamic obstacles, it is crucial to achieve robustness in dynamic environments. 
However, as seen in Table~\ref{tab:state_of_the_art_comparison}, hardware demonstrations using decentralized communication architecture (mesh network) in dynamic environments have not been tested in the literature.
For clarification, we define a dynamic environment as an environment with dynamic obstacles.
The difference between an agent and an obstacle is that an agent can make decisions based on given information.
An obstacle, on the other hand, simply follows a pre-determined trajectory regardless of what else is in the environment.

To address robustness to communication delays in dynamic environments, we propose \textbf{Robust MADER} (\RMADER{}), a decentralized, asynchronous multiagent trajectory planner capable of generating collision-free\footnote{Even if two time-parameterized trajectories coincide spatially, if the corresponding robots are not located at the same location at any given time, they would still be considered collision-free trajectories.} trajectories in dynamic environments even with communication delays. 
Note that RMADER requires a bound on the maximum communication delay to ensure safety and does not handle complete communication drops.
As shown in Table~\ref{tab:state_of_the_art_comparison}, RMADER is the first approach to demonstrate trajectory planning in dynamic environments while maintaining robustness to communication delays.
It is also worth noting that RMADER's trajectory deconfliction approach, consisting of (1) a Delay Check, (2) two-step trajectory sharing, and (3) a trajectory-storing-and-checking mechanism (detailed in Section~\ref{sec:trajectory-deconfliction}), is adaptable to any trajectory optimization algorithm that involves the sharing of trajectories and performs collision avoidance.
This work's contributions include the followings:
\begin{enumerate}
  \item An algorithm that guarantees collision-free trajectory generation even in the presence of real-world communication delays between vehicles. 
  \item Extensive simulations comparing our approach to state-of-the-art methods under communication delays that demonstrate a \textbf{100\% success rate} of collision-free trajectory generation (see Table~\ref{tab:sim_compare}).
  \item Extensive set of decentralized hardware experiments.
\end{enumerate}
This paper extends our conference paper~\cite{kondo2022robust} by including:
\begin{enumerate}
  \item Proof of RMADER's recursive feasibility (and hence collision-free) articulated within the bound of maximum communication delays, as detailed in Proposition~1. This provides an extensive elaboration of RMADER's trajectory deconfliction approach, including (1) a Delay Check, (2) two-step trajectory sharing, and (3) a novel trajectory-storing-and-checking mechanism.
  \item Comparisons with more baseline algorithms, including MADER~\cite{tordesillas_mader_2022}, EGO-Swarm~\cite{zhou_ego-swarm_2020}, EDG-Team~\cite{hou_enhanced_2022}. RMADER is shown to achieve 100\% collision avoidance and outperforms those baseline approaches.
  \item Extensive simulation studies in realistic environments with dynamic obstacles and other agents.
  \item Introduction and comparison with a new variation of RMADER (RMADER without Check) that was created to evaluate the utility of the Check step included in the RMADER trajectory deconfliction scheme.
  \item Additional hardware flight experiments using a mesh (decentralized) network with two, four, and six agents, two dynamic obstacles, and a top speed of 5.8m/s, demonstrating RMADER in a fully decentralized communication architecture. 
  This significantly extends the previous conference paper experiments~\cite{kondo2022robust} that used a centralized communication network without obstacles and vehicles restricted to slower speeds of 3.4m/s.
\end{enumerate}
The results of this paper thus show the first experimental demonstration of a multiagent trajectory planner that is resilient to communication delays and capable of handling multiple dynamic obstacles in hardware.
Note that RMADER focuses primarily on addressing communication delays in dynamic environments, and we do not address complete communication drops as described in Table~\ref{tab:multiagent_category}.



\section{Trajectory Deconfliction}\label{sec:trajectory-deconfliction}

RMADER agents plan trajectories asynchronously and broadcast the results to each other.
Each agent uses these trajectories as constraints in its optimization problem.
Assuming perfect communication with no delay, safety can be guaranteed using our previous approach presented in MADER~\cite{tordesillas_mader_2022}. 
However, this guarantee breaks down when an agent's planned trajectory is received by other agents with latency. Section~\ref{subsec:rmader_deconfliction} shows how RMADER guarantees safety even with communication delays.
We use the definitions shown in Table~\ref{tab:delaydefinitions}.

\begin{table}[h]
    \begin{centering}
    \caption{Definitions of the different delay quantities: Note that, by definition, $0\le\delayIntroduced\le\delayActual{}\le\delayActualMax{}$. See also Figs.~\ref{fig:comm_delay_in_simulation} and \ref{fig:comm_delay_on_mesh} for the actual histogram of the delays in simulation and hardware experiments.}
    \renewcommand{\arraystretch}{1.3}
    \begin{centering}
    \begin{tabular}{>{\centering\arraybackslash}m{0.1\columnwidth} >{\arraybackslash}m{0.7\columnwidth}}
    \toprule 
    \delayActual{} & Actual communication delays among agents. \tabularnewline
    \hline 
    \delayActualMax{} & \makecell[l]{Possible maximum communication delay. } \tabularnewline
    \hline 
    \delayIntroduced{} & \makecell[l]{Introduced communication delay in simulations. } \tabularnewline
    \hline 
    \delayParameter{} & \makecell[l]{Length of \DelayCheckStep{} in RMADER. \\ To guarantee safety, $\delayActualMax{}\le\delayParameter{}$ must be satisfied.} \tabularnewline
    \bottomrule
    \end{tabular}
    \par\end{centering}
    \label{tab:delaydefinitions}
    \par\end{centering}
    \vspace{-2em}
\end{table}

\subsection{Robust MADER Deconfliction} \label{subsec:rmader_deconfliction}

The primary challenge of trajectory deconfliction with communication delays in asynchronous trajectory sharing is maintaining recursive feasibility for each agent with respect to the dynamically changing/newly discovered environment from that agent's perspective and with respect to the newly received trajectories designed and published by all the other agents.
This situation is complicated by the asynchronous planning framework (new information can arrive at any time) and by the possible communication delay between when information is sent and received.
Thus, an agent cannot be confident that the newly optimized path does not conflict with recently published paths from other agents.
To achieve robustness to communication delays, we introduce (1) \DelayCheckStep{} (\DCStep{}), where an agent keeps receiving trajectories from other agents, and repeatedly checks if its newly optimized trajectory conflicts with other agents' trajectories, (2) a two-step trajectory publication scheme involving \trajOpt{} and \trajComm{} broadcasts, and (3) a trajectory-storing-and-checking mechanism.

We first explain RMADER's (1) \DCStep{} and (2) two-step trajectory publication approach by going through the pseudocode of RMADER deconfliction given in Algorithm~\ref{alg:rmader}. 
First, Agent~B runs \OptimizationStep{} (\OStep{}) to obtain \trajBOpt{} and broadcasts it if \CheckStep{} (\CStep{}) is satisfied (Line~\ref{line:broadcast_traj_B_new}).
For clarity, we use the notations \OStep{}, \CStep{}, and \DCStep{} to represent the algorithm subroutines, and \DeltaO{}, \DeltaC{}, and \DeltaDC{} to indicate the time each subroutine takes to execute, respectively.

This \CStep{} aims to determine if \trajBOpt{} has any conflicts with trajectories received in \DeltaO{}. 
Then, Agent B commits to either \trajBComm{} or \trajBOpt{} based on the following rule: if \DCStep{} detects conflicts, Agent~B commits to \trajBComm{} (Line~\ref{line:DC_not_satisfied}), and if \DCStep{} detects no conflicts, \trajBOpt{} (Line~\ref{line:DC_satisfied}).
This committed trajectory is then broadcast to the other agents (Line~\ref{line:broadcast_traj_B}). 
Fig.~\ref{fig:rmader_deconfliction} shows how RMADER deals with communication delays, where Agent A's \OStep/\CStep/\DCStep{} are \OStepA/\CStepA/\DCStepA, respectively, and Agent B's are \OStepB/\CStepB/\DCStepB. 
Similarly, Agent A's \DeltaO{}/\DeltaC{}/\DeltaDC{} are \DeltaOA{}/\DeltaCA{}/\DeltaDCA{}, and Agent B's are \DeltaOB{}/\DeltaCB{}/\DeltaDCB{}.

Note that \DeltaDC{} is either less than or equal to a predetermined parameter \delayParameter{} ($\geq$ \delayActualMax{}), and \DCStep{} operates as follows:
\begin{itemize}
    \item If a potential collision is found, \DCStep{} terminates. In this case, \DeltaDC{} is less than \delayParameter{}. The agent, which is currently following the previous trajectory that is guaranteed to be collision-free, will discard the recently optimized trajectory and initiate a new \OStep{}. In other words, in case of early termination of \DCStep{}, no new trajectory is committed, and therefore, safety is guaranteed.
    \item Otherwise, \DCStep{} terminates after \delayParameter{}, and the agent commits to the newly optimized trajectory. This newly optimized trajectory is guaranteed to be safe since no possible collisions are detected during \DeltaDC{}.
\end{itemize}
\DCStep{} is designed to detect potential collisions at the earliest opportunity. 
Rapid detection and new trajectory generation are crucial for ensuring fast and safe replanning during real-time flights.
Note that the re-planned trajectory is smooth and ensures third-order parametric continuity during the transition from the old to the new trajectory.
The initial position, velocity, and acceleration of the new trajectory are constrained to be smooth and continuous with respect to the previous trajectory.

\begin{figure}[h]
  \centering
  \begin{centering}      
  \resizebox{0.9\columnwidth}{!}{%
       \begin{tikzpicture}
       [
        greenbox/.style={shape=rectangle, fill=opt_color, draw=black},
        bluebox/.style={shape=rectangle, fill=check_color, draw=black},
         yellowbox/.style={shape=rectangle, fill=delaycheck_color, draw=black},
        ]
        
        \newcommand\Ay{2.5}
        \newcommand\Axo{1}
        \newcommand\Axc{3}
        \newcommand\Axr{4}
        \newcommand\Axe{5.5}
        
        \newcommand\By{0.7}
        \newcommand\Bxo{2.0}
        \newcommand\Bxc{4.25}
        \newcommand\Bxr{5.1}
        \newcommand\Bxe{6.6}
        
            \node[text=red] at (0.5,\Ay+0.2) {\scriptsize Agent A};
            \filldraw[fill=delaycheck_color, draw=black, opacity=0.2] (0,\Ay) rectangle (\Axo,\Ay-0.3);
            \filldraw[thick, fill=opt_color, draw=black] (\Axo,\Ay) rectangle (\Axc,\Ay-0.3);
            \filldraw[thick, fill=check_color, draw=black] (\Axc, \Ay) rectangle (\Axr, \Ay-0.3);
            \filldraw[thick, fill=delaycheck_color, draw=black] (\Axr, \Ay) rectangle (\Axe, \Ay-0.3);
            \filldraw[fill=opt_color, draw=black, opacity=0.2] (\Axe, \Ay) rectangle (\Axe+1.5, \Ay-0.3);
            \filldraw[fill=check_color, draw=black, opacity=0.2] (\Axe+1.5, \Ay) rectangle (\columnwidth, \Ay-0.3);
            \node[text=blue] at (0.5,\By+0.2) {\scriptsize Agent B};
            \filldraw[fill=check_color, draw=black, opacity=0.2] (0,\By) rectangle (\Bxo-1.5,\By-0.3);
            \filldraw[fill=delaycheck_color, draw=black, opacity=0.2] (\Bxo-1.5,\By) rectangle (\Bxo,\By-0.3);
            \filldraw[thick, fill=opt_color, draw=black] (\Bxo,\By) rectangle (\Bxc,\By-0.3);
            \filldraw[thick, fill=check_color, draw=black] (\Bxc, \By) rectangle (\Bxr, \By-0.3);
            \filldraw[thick, fill=delaycheck_color, draw=black] (\Bxr, \By) rectangle (\Bxe, \By-0.3);
            \filldraw[fill=opt_color, draw=black, opacity=0.2] (\Bxe, \By) rectangle (\columnwidth, \By-0.3);
        
        \draw[thick, densely dotted] (\Axr,-0.6) -- (\Axr,\Ay-0.3) node[] at (\Axr, -0.85) {\tiny t\textsubscript{traj\textsubscript{A\textsubscript{opt}}}};
        \draw[thick, densely dotted] (\Axe,-0.6) -- (\Axe,\Ay-0.3) node[] at (\Axe, -0.85) {\tiny t\textsubscript{traj\textsubscript{A\textsubscript{comm}}}};
            
        \draw[thick,->] (0,-0.6) -- (\columnwidth,-0.6) node[anchor=north east] {time};
        
        \draw[thick, ->, draw=red] (\Axr,\Ay-0.3) -- (\Axr,\Ay-1.2) node[midway,fill=white, text=red] {\tiny \trajAOpt{}};
        \draw[thick, ->, draw=red] (\Axe,\Ay-0.3) -- (\Axe,\Ay-1.2) node[midway,fill=white, text=red] {\tiny \trajAComm{}};
        \draw[thick, <-, draw=red] (\Bxc-0.15,\By) -- (\Bxc-0.15,\By+0.3)  node[anchor=south, fill=white, fill opacity=0.8, text opacity=1, text=black] {\tiny case 1};
        \draw[thick, <-, draw=red] (\Bxc+0.35,\By) -- (\Bxc+0.35,\By+0.3) node[anchor=south,text=black] {\tiny case 2};
        \draw[thick, <-, draw=red] (\Bxr+0.1,\By) -- (\Bxr+0.1,\By+0.3) node[anchor=south,text=black] {\tiny case 3};
        \draw[thick, <-, draw=red] (\Bxe+0.4,\By) -- (\Bxe+0.4,\By+0.3) node[anchor=south,text=black] {\tiny \makecell{case 4 \\ (not occur)}};
        \draw[thick, <->, draw=black] (\Axe,-0.2) -- (\columnwidth,-0.2) node[midway, anchor=south, text=black] {\tiny \textbf{\textcolor{red}{\trajAOpt{}} will not arrive after \DCStepA{}}};
        
        \node[font=\bfseries,right] at (\Axo,\Ay-0.15) {\tiny O\textsubscript{A}};
        \node[font=\bfseries,right] at (\Axc,\Ay-0.15) {\tiny C\textsubscript{A}};
        \node[font=\bfseries,right] at (\Axr,\Ay-0.15) {\tiny DC\textsubscript{A}};

        \node[font=\bfseries,right] at (\Bxo,\By-0.15) {\tiny O\textsubscript{B}};
        \node[font=\bfseries,right] at (\Bxc,\By-0.15) {\tiny C\textsubscript{B}};
        \node[font=\bfseries,right] at (\Bxr,\By-0.15) {\tiny DC\textsubscript{B}};
        
        
        \node[color=gray] at (0.5,\Ay-0.15) {\scriptsize Prev. iter.};
        \node[color=gray] at (0.95\columnwidth,\Ay-0.15) {\scriptsize Next iter.};
        \node[color=gray] at (0.5,\By-0.15) {\scriptsize Prev. iter.};
        \node[color=gray] at (0.95\columnwidth,\By-0.15) {\scriptsize Next iter.};
        
    \end{tikzpicture}
  }
    
  \captionof{figure}{RMADER deconfliction: After \CStepA{}, Agent A keeps executing \trajAComm{}, while checking potential collisions of newly optimized trajectory, \trajAOpt{}. This is because \trajAOpt{} might have conflicts due to communication delays and need to be checked in \DeltaDCA{}, and \trajAComm{} is ensured to be collision-free. If collisions are detected in either \DeltaCA{} or \DeltaDCA{}, Agent A keeps executing \trajAComm{}. If \DCStepA{} does not detect collisions, Agent A broadcasts and starts implementing \trajAOpt{} (i.e., $\text{\trajAComm{}}\leftarrow \text{\trajAOpt{}}$).
    \label{fig:rmader_deconfliction}
    }
    \end{centering}
\end{figure}
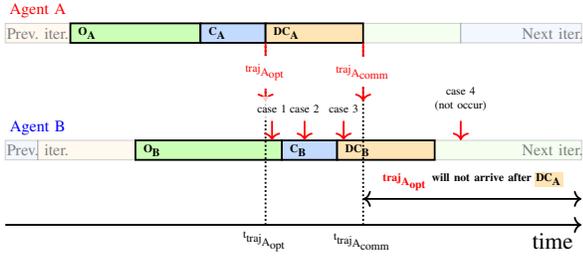

\begin{algorithm}
    \begin{algorithmic}[1]
    \Require \trajAComm{}, a feasible trajectory
    \While{not goal reached}
        \State \trajBOpt{} $=$ \textproc{\OptimizationStep{}()} \label{line:traj_opt}
        \If{\textproc{\CheckStep}(\trajBOpt{}) $==$ False} 
            \State Go to Line 2
        \EndIf
        \State Broadcast \trajBOpt{} \label{line:broadcast_traj_B_new}
        \If{\textproc{\DelayCheckStep{}}(\trajBOpt{}) $==$ False} 
            \State Keeps executing \trajBComm{} and go to Line~\ref{line:broadcast_traj_B}\label{line:DC_not_satisfied}
        \EndIf
        \State \trajBComm{} $\leftarrow$ \trajBOpt{} \label{line:DC_satisfied}
        \State Broadcast \trajBComm{} \label{line:broadcast_traj_B}
    \EndWhile
    \end{algorithmic}
    \caption{Robust MADER - Agent B}
    \label{alg:rmader}
\end{algorithm}

\newcommand{\QB}{\ensuremath{\mathcal{Q}_B}}
We now describe RMADER's trajectory-storing-and-checking approach to demonstrate that, at any given time, an agent's committed trajectories are checked against the previously committed, \trajComm{}, and newly optimized trajectories, \trajOpt{}, of other agents that have been published.

Agent B stores the trajectories received from other agents in a set \QB{}. 
Fig.~\ref{fig:QB_definition} illustrates how Agent B stores trajectories from Agent A.
Note Agent B stores Agent A's committed (and known-to-be-feasible) \trajAComm{} and the new optimized trajectory \trajAOpt{}, and Agent B uses \QB{} in \OStepB{}, \CStepB{}, and \DCStepB{} to check for collisions.
Agent B must similarly keep track of all combinations of paths from all of the nearby agents.
Furthermore, note that, throughout re-planning, \QB{} is updated as messages are received.\footnote{To be precise, due to avoiding race conditions, while we check potential collisions against trajectories stored in \QB{}, we cannot update \QB{}. However, we put the \QB{} update operation in the queue as soon as messages arrive, which results in \QB{} updates between \OStepB{}, \CStepB{}, and \DCStepB{}, as well as, throughout \DeltaDCB{} since \DCStep{} is a series of \CStep{}.}
Agent B generates an optimal trajectory, \trajBOpt{}, in \DeltaOB{}, using all the trajectories stored in \QB{}($t_{\text{opt-start}}$) as constraints, where $t_{\text{opt-start}}$ is the start time of \OStepB{}.
Subsequently, in \DeltaCB{}, Agent B checks \trajBOpt{} for potential collisions against \QB{}($t_{\text{check-start}}$) as constraints, where $t_{\text{check-start}}$ is the start time of \CStepB{}.
Finally, in \DeltaDCB{}, Agent B repeatedly checks for collisions against \QB{} while simultaneously updating it. 
This trajectory-storing-and-checking approach ensures that agents' committed trajectories are verified against other agents' previously committed and newly optimized trajectories that have already been published.

\begin{figure}[h]
        \centering
        \includegraphics[width=0.6\columnwidth]{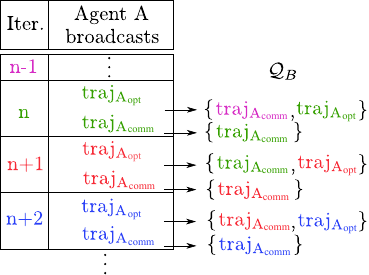}
        \caption{RMADER's trajectory storing and checking approach: When Agent B receives Agent A's committed trajectory, \trajAComm{}, Agent B clears all the previously received Agent A's trajectories in \QB{} and only stores the committed trajectories. When Agent B receives Agent A's new optimized trajectory, \trajAOpt{}, Agent B adds it to \QB{} while keeping Agent A's committed trajectory. Note that the pair of curly brackets \{\} is \QB{}, and it gets updated as Agent B receives trajectories. ~\label{fig:QB_definition}}
    \vspace{-2em}
\end{figure}

To demonstrate that RMADER's deconfliction strategy guarantees collision safety in all the possible cases of message exchange between agents, we provide RMADER's collision-free guarantee in  Proposition~\ref{prop:rmader_deconfliction_theorem}.

\newtheorem{prop}{Proposition}
\begin{prop}
\label{prop:rmader_deconfliction_theorem}
Given $\delayActualMax{}\le\delayParameter{}$, RMADER is guaranteed to be collision-free under communication delay.
\end{prop}

\begin{proof}
    First, note that in decentralized asynchronous trajectory planning, ensuring collision safety between any two agents is a sufficient condition for guaranteeing multiagent collision safety.
    This is because decentralized asynchronous multiagent trajectory deconflictions consist of sets of two-agent trajectory deconflictions.
    This allows us to apply two-agent deconfliction analyses to multiagent deconfliction.
    Also, it is crucial to note that collision safety between two agents is guaranteed if at least one agent is aware of the other's trajectory at any given time.
    In RMADER, this condition is satisfied due to RMADER's trajectory-storing-and-checking approach and $\delayActualMax{}\le\delayParameter{}$.

    RMADER's (1) \DCStep{}, (2) two-step trajectory sharing, and (3) trajectory-storing-and-checking approach are designed to ensure safe trajectory deconfliction in an asynchronous setting. 
    To prove RMADER's collision-free guarantee, we enumerate the 12 possible cases of trajectory deconfliction between two agents.
    Note that when one agent sends its trajectory to the other, the other agent will receive it in either \DeltaO{}, \DeltaC{}, \DeltaDC{}, or in the following iterations. 
    For clarity, we define the time Agent A publishes its newly optimized trajectory as \tApub, and the time Agent B receives it as \tBrec.
    All 12 possible cases are illustrated in Table~\ref{tab:rmader_possible_cases}.
    For example, case 3 denotes that Agent A publishes in \DeltaOB{}, and Agent B receives in \DeltaDCB{}.
    Note that cases 1-4 correspond to cases 1-4 in Fig.~\ref{fig:rmader_deconfliction}, and cases 5-12 are illustrated in Figs.~\ref{fig:rmader_deconfliction-case-5-8} and \ref{fig:rmader_deconfliction-case-9-12} in Appendix.
    
    Note that the lengths of \DeltaO{} and \DeltaC{} vary with each planning iteration. 
    However, this does not affect RMADER's guarantee of being collision-free, as the proof does not make any assumptions regarding the lengths of \DeltaO{} and \DeltaC{}.
    
    \begin{table}[h]
    \caption{\centering RMADER Trajectory Deconfliction Cases}
    \label{tab:rmader_possible_cases}
    \renewcommand{\arraystretch}{1.2}
    \begin{centering}
    \resizebox{\columnwidth}{!}{
    \begin{tabular}{ c c || c c c c }
    \toprule
     & & \multicolumn{4}{c}{\tBrec}\tabularnewline
     & & \DeltaOB{} & \DeltaCB{} & \DeltaDCB{} & After \DeltaDCB{} \tabularnewline
    \midrule
    \midrule
    \multirow{3}[1]{*}{\tApub} & \DeltaOB{} & case 1 & case 2 & case 3 & case 4 \tabularnewline
     & \DeltaCB{} & case 5 & case 6 & case 7 & case 8 \tabularnewline
     & \DeltaDCB{} & case 9 & case 10 & case 11 & case 12 \tabularnewline
    \bottomrule
    \end{tabular}}
    \end{centering}
    \end{table}
    
    \begin{table}[h]
    \caption{\centering RMADER Collision Detection Cases: Either Agent A or B will detect potential collisions in all 12 possible cases.}
    \label{tab:rmader_detection_cases}
    \begin{mdframed}[linecolor=white,linewidth=0.1pt]
    \renewcommand{\arraystretch}{0.6}
    \centering
    \begin{centering}
    \resizebox{0.8\columnwidth}{!}{
    \centering
    \begin{tabular}{ c c c }
    \toprule
    Case & \makecell{Given $\delayActualMax{}\le\delayParameter{}$, \\ could this occur?} & When Agent B detects \tabularnewline
    \midrule
    1    & \YesGreen{}                                                              & \DeltaCB{} \tabularnewline
    \midrule
    2    & \YesGreen{}                                                                & \DeltaDCB{} \tabularnewline
    \midrule
    3    & \YesGreen{}                                                                & \DeltaDCB{} \tabularnewline
    \midrule
    4    & \NoRed{}                                                               & N/A \tabularnewline
    \midrule
    5    & \YesGreen{}                                                                & \OStepB{} \tabularnewline
    \midrule
    6    & \YesGreen{}                                                                & \DeltaDCB{} \tabularnewline
    \midrule
    7    & \YesGreen{}                                                                & \DeltaDCB{} \tabularnewline
    \midrule
    8    & \NoRed{}                                                               & N/A \tabularnewline
    \midrule
    9    & \YesGreen{}                                                                & \OStepB{} \tabularnewline
    \midrule
    10   & \YesGreen{}                                                                & \CStepB{} \tabularnewline
    \midrule
    11   & \YesGreen{}                                                                & \DeltaDCB{} \tabularnewline
    \midrule
    12   & \YesGreen{}                                                                & Agent A detects.\footnotemark[7]\tabularnewline
    \bottomrule
    \end{tabular}}
    \end{centering}
    \vspace*{0.5em}
    \end{mdframed}
    \footnotesize{$^7$ \!\!\! Given $\delayActualMax{}\le\delayParameter{}$, \trajBOpt{} will arrive before the end of \DCStepA{}, and therefore, Agent A can detect potential collisions before committing to its trajectory (See Fig.~\ref{fig:rmader_deconfliction-case-9-12} in Appendix). 
    In other words, Agent A now deconflicts trajectories.
    Note that this potential collision detection during \DeltaDCA{} is reliably achieved irrespective of the lengths of \DeltaOA{} and \DeltaCA{}.}
    \vspace{-1em}
    \end{table}

    Table ~\ref{tab:rmader_detection_cases} shows, in all the possible cases, either agent can detect potential collisions (see Fig.~\ref{fig:rmader_deconfliction} for cases 1-4, Fig.~\ref{fig:rmader_deconfliction-case-5-8} for cases 5-8, and Fig.~\ref{fig:rmader_deconfliction-case-9-12} for cases 9-12).
    Therefore RMADER is guaranteed to be collision-free as long as $\delayActualMax{}\le\delayParameter{}$ holds.
\end{proof}





\begin{algorithm}
    \newcommand{\algorithmicbreak}{\textbf{break}}
    \begin{algorithmic}[1] 
    \Function {\textproc{\DelayCheckStep{}}}{\trajBOpt{}}
        \For {\delayParameter{}}
            \If{\textproc{\CheckStep}(\trajBOpt{}) $==$ False} 
                \State Discard \trajBOpt{}
                \State \Return False
             \EndIf
         \EndFor
    \State \Return True
    \EndFunction
    \end{algorithmic}
    \caption{Delay Check - Agent B}\label{alg:pess_delaycheck}
\end{algorithm}


\subsection{RMADER without Check}\label{sec:rmader-wo-check}

This section introduces a new variation of RMADER and evaluates \CheckStep{} in RMADER's deconfliction.
As demonstrated in Algorithm~\ref{alg:pess_delaycheck}, \DCStep{} repeatedly executes \CStep{}, and Proposition~\ref{prop:rmader_deconfliction_theorem} have no assumptions regarding the length of \DeltaC{}, and therefore, we can guarantee the safety of RMADER even without \CStep{} as long as \NeccessaryCond{}.
We therefore remove \CStep{} and introduce a new variation of RMADER \textemdash RMADER without Check. 
Since it does not go through \CStep{}, it will replan, publish, and potentially generate trajectory faster.
However, this could increase the number of rejections of newly optimized trajectories because it is published before it is checked for potential collisions with trajectories received in \DeltaO{}, leading to more conservative results.
We introduced this variation to evaluate the need for \CStep{}, as it is not necessary to guarantee safety, and removing \CStep{} potentially achieves faster replanning. 
Note that even if we skip \CStep{} and publish \trajOpt{}, \trajOpt{} is not committed yet, and therefore, collision-free guarantee is ensured. 
Fig.~\ref{fig:wo_check_rmader_deconfliction} illustrates this scheme's trajectory deconfliction: 

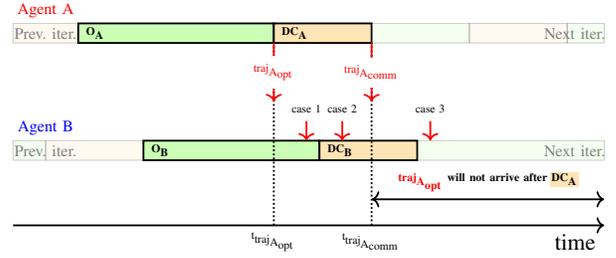
\begin{figure}[h]
\centering
\begin{minipage}{.5\textwidth}
  \begin{centering}
  \resizebox{0.9\textwidth}{!}{%
       \begin{tikzpicture}
       [
        greenbox/.style={shape=rectangle, fill=opt_color, draw=black},
        bluebox/.style={shape=rectangle, fill=check_color, draw=black},
         yellowbox/.style={shape=rectangle, fill=delaycheck_color, draw=black},
        ]
        
        \newcommand\Ay{2.5}
        \newcommand\Axo{1}
        \newcommand\Axc{3}
        \newcommand\Axr{4}
        \newcommand\Axe{5.5}
        
        \newcommand\By{0.7}
        \newcommand\Bxo{2.0}
        \newcommand\Bxc{4.7}
        \newcommand\Bxr{6.0}
        \newcommand\Bxe{7.5}
        
            \node[text=red] at (0.5,\Ay+0.2) {\scriptsize Agent A};
            \filldraw[fill=delaycheck_color, draw=black, opacity=0.2] (0,\Ay) rectangle (\Axo,\Ay-0.3);
            \filldraw[thick, fill=opt_color, draw=black] (\Axo,\Ay) rectangle (\Axr,\Ay-0.3);
            \filldraw[thick, fill=delaycheck_color, draw=black] (\Axr, \Ay) rectangle (\Axe, \Ay-0.3);
            \filldraw[fill=opt_color, draw=black, opacity=0.2] (\Axe, \Ay) rectangle (\Axe+1.5, \Ay-0.3);
            \filldraw[fill=delaycheck_color, draw=black, opacity=0.2] (\Axe+1.5, \Ay) rectangle (\Axe+3, \Ay-0.3);
            \filldraw[fill=opt_color, draw=black, opacity=0.2] (\Axe+3, \Ay) rectangle (\columnwidth, \Ay-0.3);
            \node[text=blue] at (0.5,\By+0.2) {\scriptsize Agent B};
            \filldraw[fill=opt_color, draw=black, opacity=0.2] (0,\By) rectangle (\Bxo-1.5,\By-0.3);
            \filldraw[fill=delaycheck_color, draw=black, opacity=0.2] (\Bxo-1.5,\By) rectangle (\Bxo,\By-0.3);
            \filldraw[thick, fill=opt_color, draw=black] (\Bxo,\By) rectangle (\Bxc,\By-0.3);
            \filldraw[thick, fill=delaycheck_color, draw=black] (\Bxc, \By) rectangle (\Bxc+1.5, \By-0.3);
            \filldraw[fill=opt_color, draw=black, opacity=0.2] (\Bxc+1.5, \By) rectangle (\columnwidth, \By-0.3);
        
        \draw[thick, densely dotted] (\Axr,-0.6) -- (\Axr,\Ay-0.3) node[] at (\Axr, -0.85) {\tiny t\textsubscript{traj\textsubscript{A\textsubscript{opt}}}};
        \draw[thick, densely dotted] (\Axe,-0.6) -- (\Axe,\Ay-0.3) node[] at (\Axe, -0.85) {\tiny t\textsubscript{traj\textsubscript{A\textsubscript{comm}}}};
            
        \draw[thick,->] (0,-0.6) -- (\columnwidth,-0.6) node[anchor=north east] {time};
        
        \draw[thick, ->, draw=red] (\Axr,\Ay-0.3) -- (\Axr,\Ay-1.2) node[midway,fill=white, text=red] {\tiny \trajAOpt{}};
        \draw[thick, ->, draw=red] (\Axe,\Ay-0.3) -- (\Axe,\Ay-1.2) node[midway,fill=white, text=red] {\tiny \trajAComm{}};
        \draw[thick, <-, draw=red] (\Bxc-0.2,\By) -- (\Bxc-0.2,\By+0.3)  node[anchor=south,text=black] {\tiny case 1};
        \draw[thick, <-, draw=red] (\Bxc+0.35,\By) -- (\Bxc+0.35,\By+0.3) node[anchor=south,text=black] {\tiny case 2};
        \draw[thick, <-, draw=red] (\Bxr+0.4,\By) -- (\Bxr+0.4,\By+0.3) node[anchor=south,text=black] {\tiny case 3};
        \draw[thick, <->, draw=black] (\Axe,-0.2) -- (\columnwidth,-0.2) node[midway, anchor=south, text=black] {\tiny \textbf{\textcolor{red}{\trajAOpt{}} will not arrive after \DCStepA{}}};
        
        \node[font=\bfseries,right] at (\Axo,\Ay-0.15) {\tiny O\textsubscript{A}};
        \node[font=\bfseries,right] at (\Axr,\Ay-0.15) {\tiny DC\textsubscript{A}};

        \node[font=\bfseries,right] at (\Bxo,\By-0.15) {\tiny O\textsubscript{B}};
        \node[font=\bfseries,right] at (\Bxc,\By-0.15) {\tiny DC\textsubscript{B}};
        
        
        \node[color=gray] at (0.5,\Ay-0.15) {\scriptsize Prev. iter.};
        \node[color=gray] at (0.95\columnwidth,\Ay-0.15) {\scriptsize Next iter.};
        \node[color=gray] at (0.5,\By-0.15) {\scriptsize Prev. iter.};
        \node[color=gray] at (0.95\columnwidth,\By-0.15) {\scriptsize Next iter.};
        
    \end{tikzpicture}
    }
  \captionof{figure}{RMADER deconfliction without \CStep{}: After \OStepA, Agent A publishes its newly optimized trajectory \trajAOpt{}, while it keeps executing \trajAComm{}. As with Fig.~\ref{fig:rmader_deconfliction}, the agent checks conflicts in \DeltaDCA{}, and if collisions are detected in \DeltaDCA{}, Agent A continues on executing \trajAComm{}. In case \DCStepA{} does not detect collisions, Agent A broadcasts and starts implementing \trajAOpt{} (i.e., $\text{\trajAComm{}}\leftarrow \text{\trajAOpt{}}$).\label{fig:wo_check_rmader_deconfliction}}
  \end{centering}
\end{minipage}
\end{figure}


\section{Simulation Results}\label{sec:sim}
\subsection{Benchmark Studies}\label{subsec:benchmark}

\colorlet{color0ms}{blue}
\colorlet{color50ms}{OliveGreen}
\colorlet{color100ms}{RawSienna}
\colorlet{color200ms}{Magenta}
\colorlet{color300ms}{Orange}

\newcommand{\threec}[3]{\ensuremath{\textcolor{color0ms}{#1}|\textcolor{color50ms}{#2}|\textcolor{color100ms}{#3}}}

\newcommand{\fivec}[5]{\begin{centering}\ensuremath{\hspace{0.7em} \textcolor{color0ms}{\scriptsize #1}|\textcolor{color50ms}{\scriptsize #2} \newline \hspace{-5em} \textcolor{color100ms}{ \scriptsize #3}|\textcolor{color200ms}{\scriptsize #4}|\textcolor{color300ms}{\scriptsize{#5}}}\end{centering}}

\newcommand{\fivecfordeadlock}[5]{\begin{centering}\ensuremath{\hspace{0.55em}\textcolor{color0ms}{#1}|\textcolor{color50ms}{#2} \newline \hspace{-5em} \textcolor{color100ms}{#3}|\textcolor{color200ms}{#4}|\textcolor{color300ms}{#5}}\end{centering}}

\newcommand{\threecb}[3]{\ensuremath{\textcolor{color0ms}{\textbf{#1}}|\textcolor{color50ms}{\textbf{#2}}|\textcolor{color100ms}{\textbf{#3}}}}
\newcommand{\fivecb}[5]{\begin{centering}\ensuremath{\hspace{0.7em}\textcolor{color0ms}{\scriptsize \textbf{#1}}|\textcolor{color50ms}{\scriptsize \textbf{#2}} \newline \textcolor{color100ms}{\scriptsize \textbf{#3}}|\textcolor{color200ms}{\scriptsize \textbf{#4}}|\textcolor{color300ms}{\scriptsize \textbf{#5}}}\end{centering}}

\begin{table}[h]
    \scriptsize
    \caption{\centering Benchmark Studies: Cases \textcolor{color0ms}{$\delayIntroduced{}=0$~ms}, \textcolor{color50ms}{$\delayIntroduced{}=50$~ms}, \textcolor{color100ms}{$\delayIntroduced{}=100$~ms}, \textcolor{color200ms}{$\delayIntroduced{}=200$~ms}, and \textcolor{color300ms}{$\delayIntroduced{}=300$~ms}.}
    \label{tab:sim_compare}    
    \centering
    \resizebox{\columnwidth}{!}{
    \begin{tabular}{>{\centering\arraybackslash}m{0.07\textwidth} >{\centering\arraybackslash}m{0.05\textwidth} >{\centering\arraybackslash}m{0.08\textwidth} >{\centering\arraybackslash}m{0.08\textwidth} >{\centering\arraybackslash}m{0.06\textwidth}}
        \toprule
        \centering Method & \centering Async.? & \centering \delayParameter{} [ms] & \centering Collision-free rate [\%] & Deadlock rate [\%] \tabularnewline
        \midrule

        \RMADER{} & \YesGreen & \fivec{75}{125}{175}{275}{375} & \fivecb{100}{100}{100}{100}{100} & \fivec{1}{1}{0}{0}{0}  \tabularnewline
        \hline

        \WOCHECKRMADER{} & \YesGreen & \fivec{75}{125}{175}{275}{375} & \fivecb{100}{100}{100}{100}{100} & \fivec{0}{2}{3}{0}{5} \tabularnewline
        \hline

        \MADER{} & \YesGreen & N/A & \fivecb{95}{92}{87}{72}{70} & \fivec{0}{0}{0}{0}{0} \tabularnewline
        \hline

        \EGOswarm{} & \YesGreen & N/A & \fivecb{34}{27}{27}{48}{29} & \fivec{0}{0}{0}{0}{0} \tabularnewline
        \hline

        \EDGteam{} & \YesGreen/\NoRed & N/A & \fivecb{100}{100}{99}{99}{99} & \fivec{0}{0}{0}{0}{0}  \tabularnewline
        \bottomrule 
    \end{tabular}
    }
    \vspace{-1em}
\end{table}
\begin{figure}
    \centering
    \subcaptionbox{10-agent position exchange simulation environment.\label{fig:circle_traj_new}}{
    \begin{tikzpicture}[every text node part/.style={align=center}]
    \node {\includegraphics[width=0.45\columnwidth, height=0.38\columnwidth]{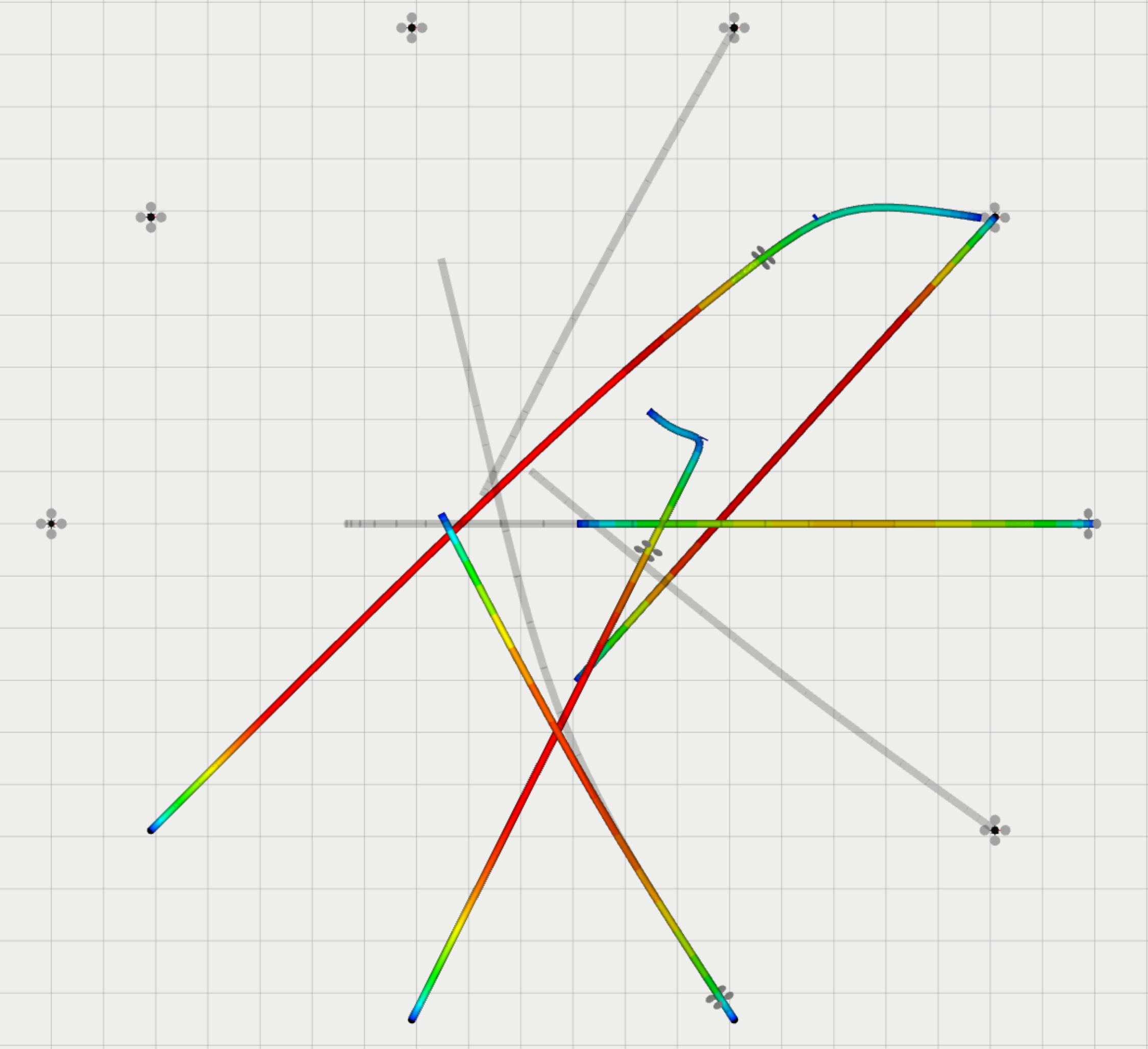}};
    \node [rectangle] at (1.2, -1.4) {\scriptsize Agent J};
    \node [circle, draw] (c) at (0.5,-1.4){};
    \node (trajjnew) [rectangle] at (-1.2, 1.2) {\scriptsize \trajJOpt{} (grey)}; 
    \draw [-stealth] (trajjnew) -- (-0.4, 0.4);
    \node (trajj) [rectangle] at (-1, -1) {\scriptsize \trajJComm{} (colored)}; 
    \draw [-stealth] (trajj) -- (-0.38, -0.4);
    \end{tikzpicture}}
    \subcaptionbox{10 RMADER agents with 10 dynamic obstacles.\label{fig:rmader_obs}}{
    \begin{tikzpicture}[every text node part/.style={align=center}]
    \node {\includegraphics[width=0.45\columnwidth, height=0.38\columnwidth]{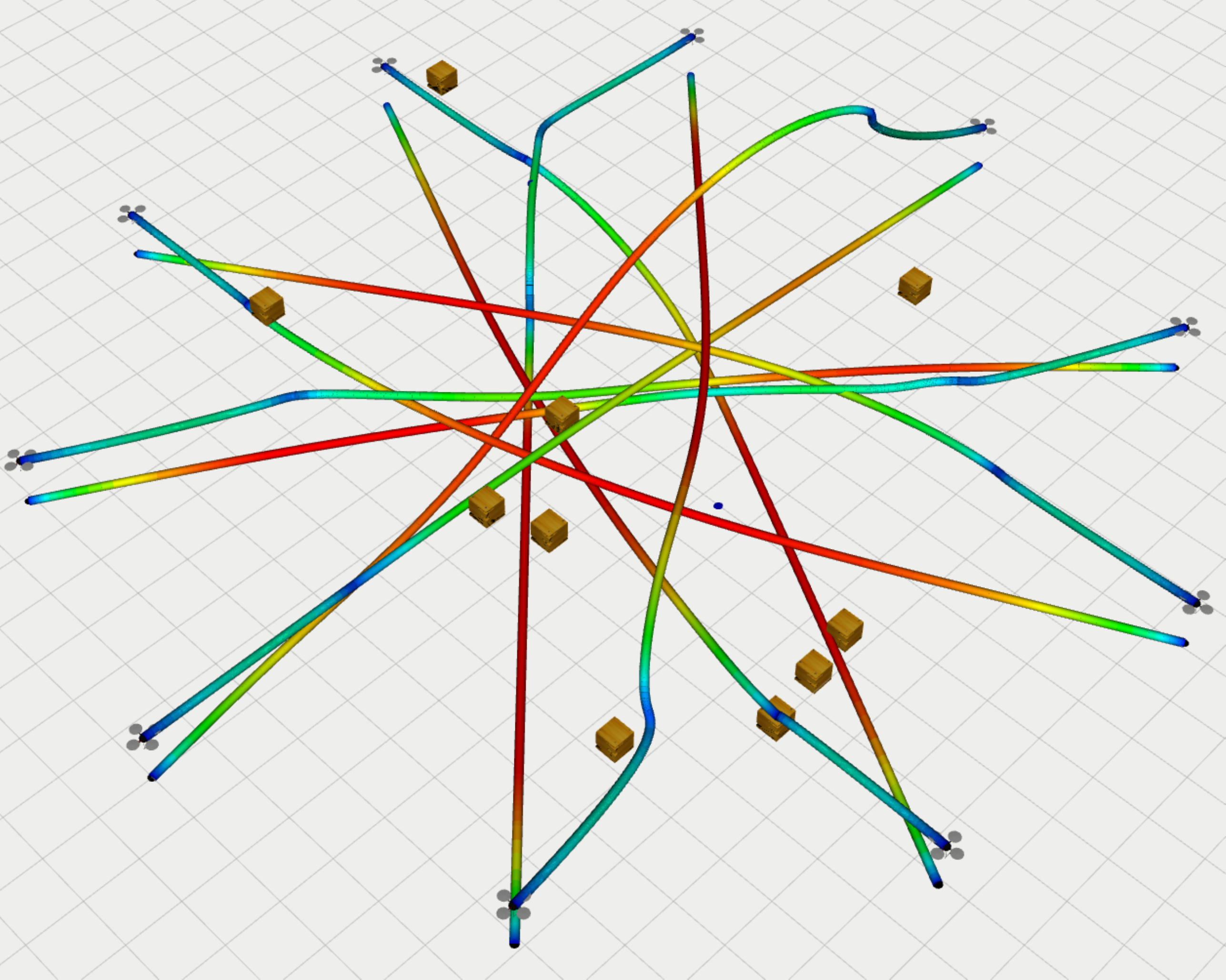}};
    \end{tikzpicture}}
    \caption{10 agents employing RMADER exchange their positions in a circle of radius \SI{20}{\m}. In the colored trajectories, red represents a high speed while blue denotes a low speed. In (a), The colored trajectory is the committed (safety-guaranteed) trajectory (\trajJComm{}), and the grey trajectory is the newly optimized trajectory (\trajJOpt{}).}
    \label{fig:rmader_sim1}
    \vspace{-1em}
\end{figure}
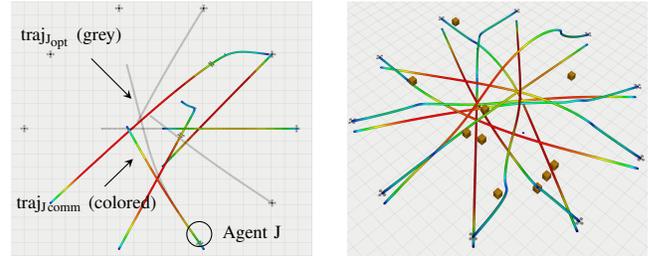
\begin{figure}
    \vspace{1em}
    \begin{centering}        
    \begin{tabular}{ccc}
      \subfloat[ \centering Collision-free Trajectory Rate \label{fig:sim_collision_free_traj_rate}]{\includegraphics[width=0.45\columnwidth]{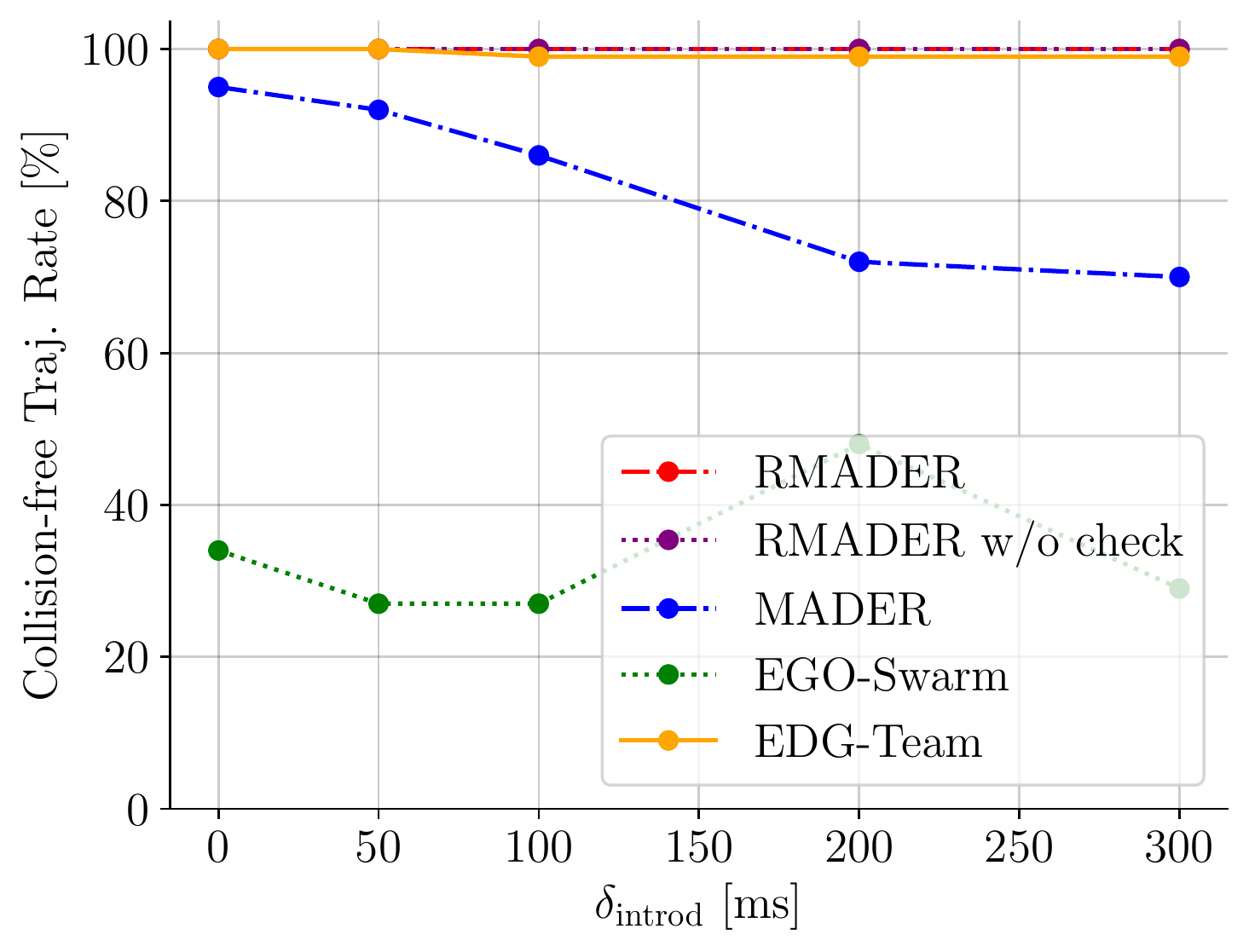}} &
      \subfloat[\centering Avg. Travel Time\label{fig:sim_completion_time}]{\includegraphics[width=0.45\columnwidth]{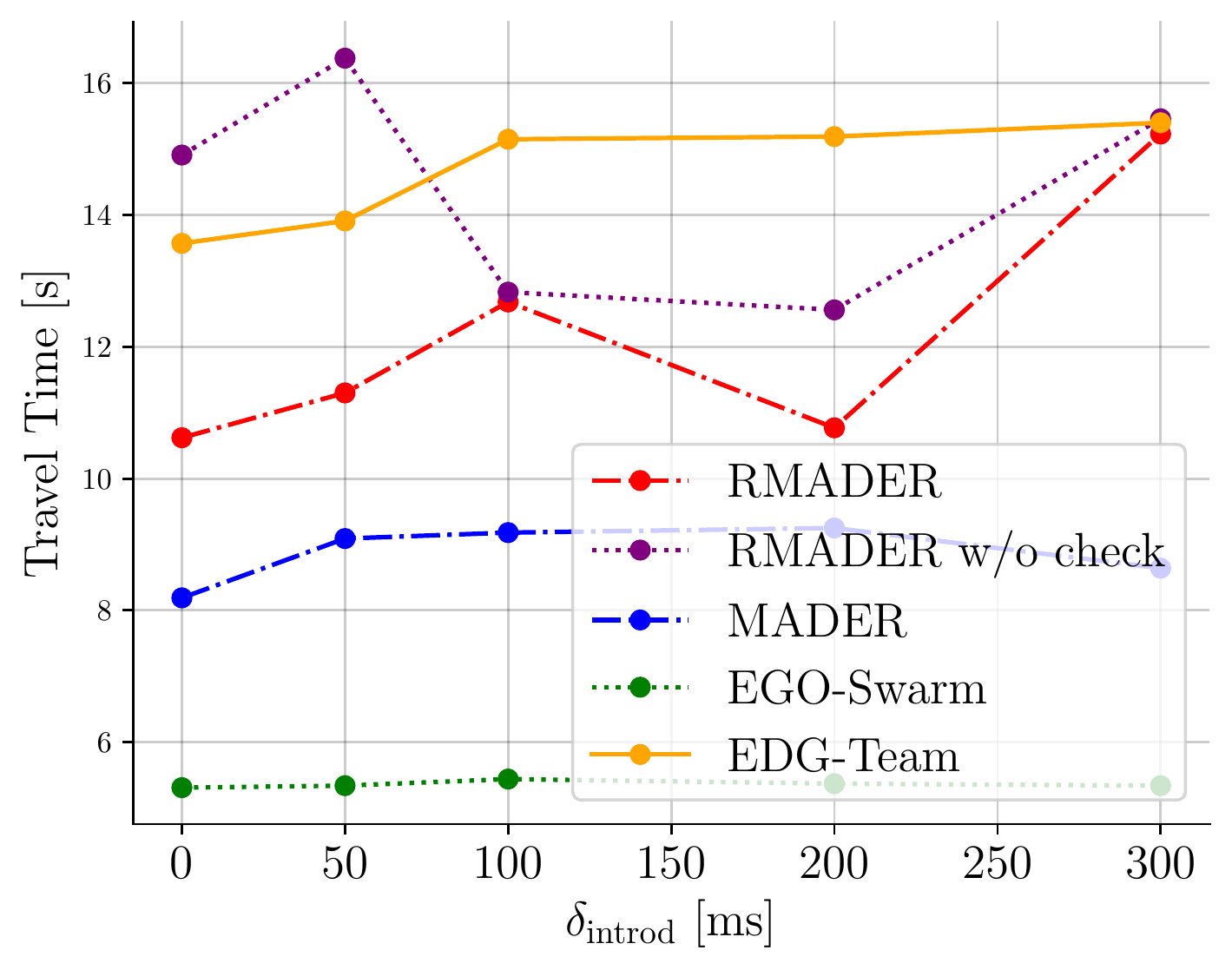}} \\
      \subfloat[\centering Number of Stops \label{fig:sim_traj_smooth_acc}]{\includegraphics[width=0.45\columnwidth]{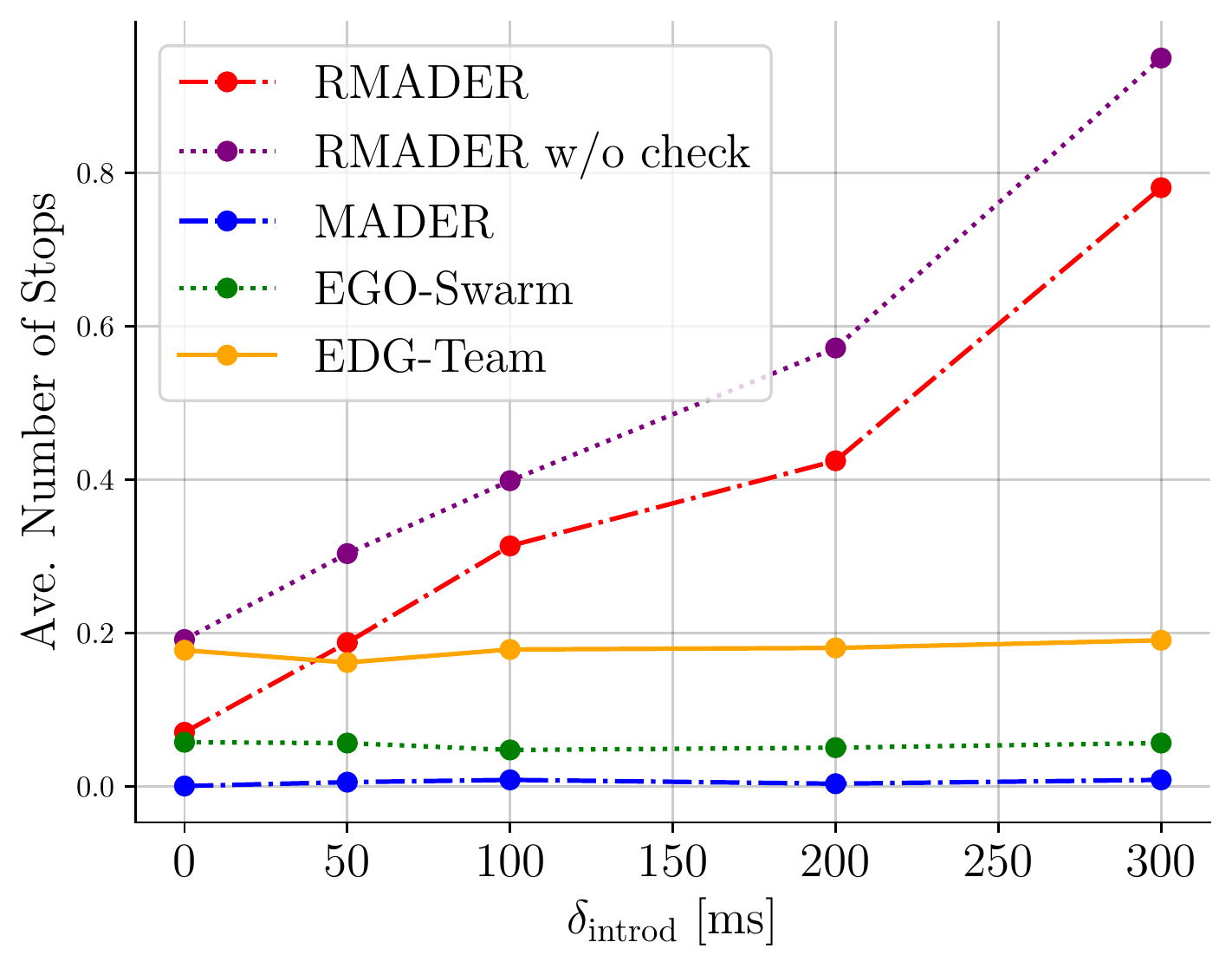}} &
      \subfloat[\centering Trajectory Smoothness (Jerk) \label{fig:sim_traj_smooth_jer}]{\includegraphics[width=0.45\columnwidth]{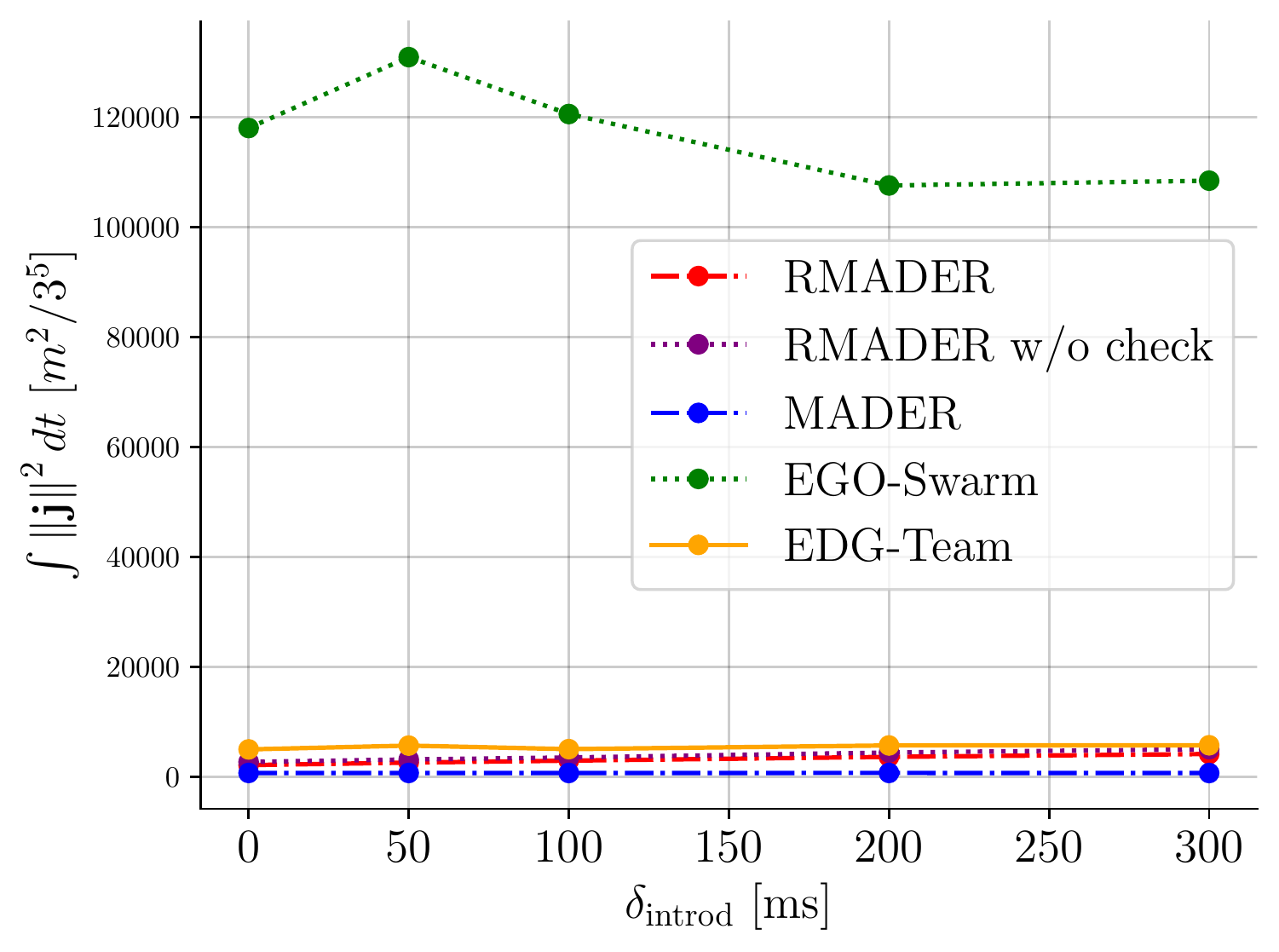}} \\
      \subfloat[\centering Average Travel Distance\label{fig:sim_total_travel_dist}]{\includegraphics[width=0.45\columnwidth]{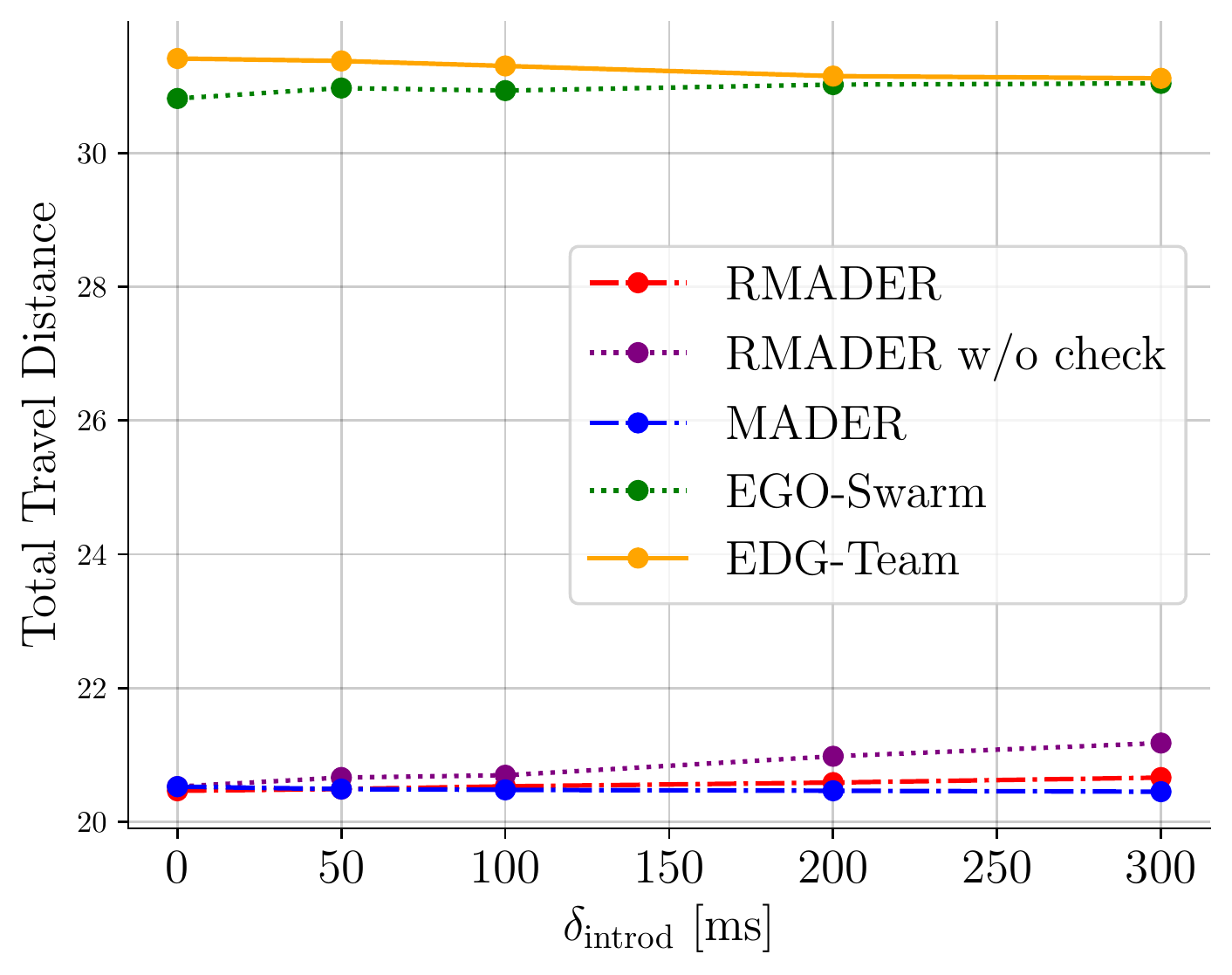}} &
      \subfloat[\centering Distribution of \delayActual{}\label{fig:comm_delay_in_simulation}]{\includegraphics[width=0.45\columnwidth]{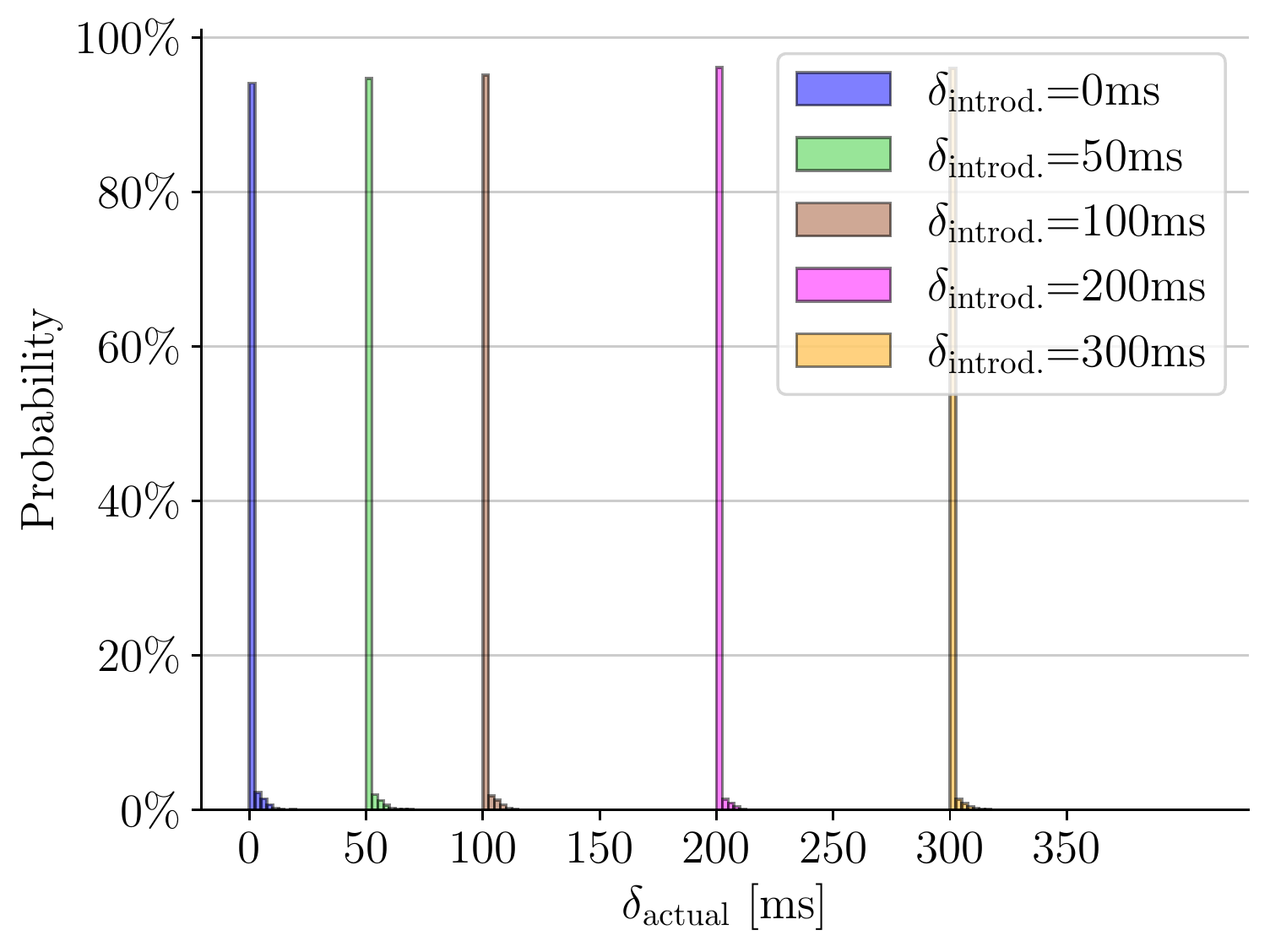}}
    \end{tabular}
    \vspace{-0.5em}
    \caption{100-simulation results.
    (a) RMADER generates collision-free trajectory at 100\%, while other state-of-the-art approaches fail when communication delays are introduced. 
    To maintain collision-free trajectory generation, An RMADER agent periodically maintains two trajectories, and other agents need to consider two trajectories as a constraint, which could lead to conservative plans---longer \emph{Travel Time} (b) and more \emph{Avg. Number of Stops} (c).
    This is a trade-off between safety and performance.
    MADER reports a few collided trajectory because \delayActual{} $>$ \SI{0}{\ms}.
    It is also important to note that although we introduced fixed communication delays, due to the computer’s computational limits, messages actually do not arrive instantly, as shown in Fig.~\ref{fig:comm_delay_in_simulation}.
    These figures differ slightly from the ones presented in \cite{kondo2022robust}---we believe the reasons are (1) RMADER's parameters are extensively tuned, (2) to introduce artificial communication delays efficiently, we improved the code for all the methods, (3) simulations with multiagent planners are stochastic.} 
    \label{fig:sim_simulation_summary}
    \end{centering}
    \vspace{-2em}
\end{figure}

We tested EGO-Swarm~\cite{zhou_ego-swarm_2020}, EDG-Team~\cite{hou_enhanced_2022}, MADER~\cite{tordesillas_mader_2022}, RMADER, and RMADER without Check, on a \texttt{Lambda} computer with AMD Threadripper 3960X 24-Core, 48-Thread. 
We introduced $0$, $50$, $100$, $200$, and \SI{300}{\ms} communication delay into every message exchanged among agents. 
We then conducted \textbf{100 simulations} with 10 agents positioned in a \SI{10}{\m} radius circle, exchanging positions diagonally as shown in Fig.~\ref{fig:rmader_sim1}.
The maximum dynamic limits are set to \SI{10}{\m/\s}, \SI{20}{\m/\s^2}, and \SI{30}{\m/\s^3}.

Table~\ref{tab:sim_compare} and Fig.~\ref{fig:sim_simulation_summary} showcase each approach's performance in simulations. 
When \NeccessaryCond{} holds, collision-free trajectory planning is guaranteed, and therefore RMADER and RMADER without Check generate \textbf{0 collisions} for all the \delayIntroduced{}, while other approaches fail to avoid collisions. 
As expected, the longer \delayIntroduced{} more collisions EGO-Swarm, EDG-Team, and MADER generate.

As for the utility of \CStep{}, Fig.~\ref{fig:sim_simulation_summary} illustrates that agents running RMADER without Check take more time to travel and experience a higher number of stops. As elaborated in Section~\ref{sec:rmader-wo-check}, this can be attributed to its potential to reject more trajectories. These results emphasize the importance of incorporating \CStep{} into RMADER's deconfliction strategy.

It is also worth mentioning that RMADER's robustness to communication delays is obtained by layers of conflict checks in \DCStep{} and agents periodically holding two trajectories, which can result in generating conservative trajectories and trading off UAV performance. 
For instance, \emph{Avg. Number of Stops} in Table~\ref{tab:sim_compare}, suggests more stoppage than other approaches, and although RMADER reports no collisions, it suffers a small number of deadlocks.

\subsection{Simulations with Dynamic Obstacles}

To verify RMADER's robustness to dynamic environments in simulation, we performed 100 simulations with 10 agents with 10 dynamic obstacles under \SI{50}{\ms} communication delays and compare its performance to MADER. The agents' size is \qtyproduct{0.05 x 0.05 x 0.05}{\m}, and the obstacles' size is \qtyproduct{0.4 x 0.4 x 0.4}{\m}, and the obstacles follow randomized trefoil trajectories. Fig.~\ref{fig:rmader_obs} presents the simulation environment, and Table~\ref{tab:sim_with_obs} illustrates RMADER successfully achieved \textbf{\SI{100}{\%} collision-free trajectory generation}. Similarly to Section\ref{subsec:benchmark}, RMADER's performance is traded off to achieve safety.

\begin{table}[!h]
    \caption{Simulations with 10 obstacles under \SI{50}{\m\s} of communication delay. RMADER successfully achieves collision-free trajectory generation with dynamic obstacles.}
    \label{tab:sim_with_obs}    
    \renewcommand{\arraystretch}{1.2}
    \centering
    \resizebox{0.8\columnwidth}{!}{
    \begin{tabular}{>
    {\centering\arraybackslash}m{0.2\columnwidth} >{\centering\arraybackslash}m{0.2\columnwidth}  >{\centering\arraybackslash}m{0.2\columnwidth}  >{\centering\arraybackslash}m{0.2\columnwidth}  >{\centering\arraybackslash}m{0.2\columnwidth} }
        \toprule
         & Collision-free rate [\%] & \centering Avg number of stops & \centering Avg Travel Time [s] & \centering Avg Travel Distance [m] \tabularnewline
        \midrule
        \textbf{MADER} & 99 & 0.056 & 10.47 & 20.45 \tabularnewline
        \hline
        \textbf{RMADER} & \textbf{100} & 0.163 & 15.81 & 20.52 \tabularnewline
        \bottomrule 
    \end{tabular}
    }
    \vspace{-1em}
\end{table}


\section{Hardware Experiments}


\begin{table}
\caption{\centering MADER on Centralized Network}
\label{tab:mader_hw_centralized}
\begin{centering}
\renewcommand{\arraystretch}{1.2}
\resizebox{0.8\columnwidth}{!}{
\begin{tabular}{ c c c c c c }
\toprule
 & \textbf{Exp. 1} & \textbf{Exp. 2} & \textbf{Exp. 3} & \textbf{Exp. 4} & \textbf{Exp.5} \tabularnewline
\midrule
\textbf{Max vel. [m/s]} & 2.7 & 2.5 & 2.7 & 2.7 & 3.0 \tabularnewline
\hline 
\textbf{Avg. travel distance [m]} & 70.3 & 67.6 & 61.1 & 61.1 & 65.2\tabularnewline
\hline
\textbf{Stop time [s]} & 1.3 & 1.3 & 2.2 & 3.4 & 4.3 \tabularnewline
\bottomrule
\end{tabular}}
\par\end{centering}
\end{table}

\begin{table}
\caption{\centering RMADER on Centralized Network}
\label{tab:rmader_hw_centralized}
\begin{centering}
\renewcommand{\arraystretch}{1.2}
\resizebox{0.8\columnwidth}{!}{
\begin{tabular}{ c c c c c c }
\toprule
 & \textbf{Exp. 6} & \textbf{Exp. 7} & \textbf{Exp. 8} & \textbf{Exp. 9} & \textbf{Exp. 10}\tabularnewline
\midrule
\textbf{Max vel. [m/s]} & 2.6 & 2.7 & 2.8 & 2.7 & 2.7\tabularnewline
\hline 
\textbf{Avg. travel distance [m]} & 45.6 & 58.2 & 58.4 & 58.3 & 54.8\tabularnewline
\hline
\textbf{Stop time [s]} & 12.3 & 10.6 & 8.1 & 9.9 & 7.4 \tabularnewline
\bottomrule
\end{tabular}}
\par\end{centering}
\end{table}

\begin{figure}[h]
    \centering
    \subfloat[\scriptsize Distribution of \delayActual{} in mesh (decentralized) and centralized network.]{
    \includegraphics[width=0.4\columnwidth]{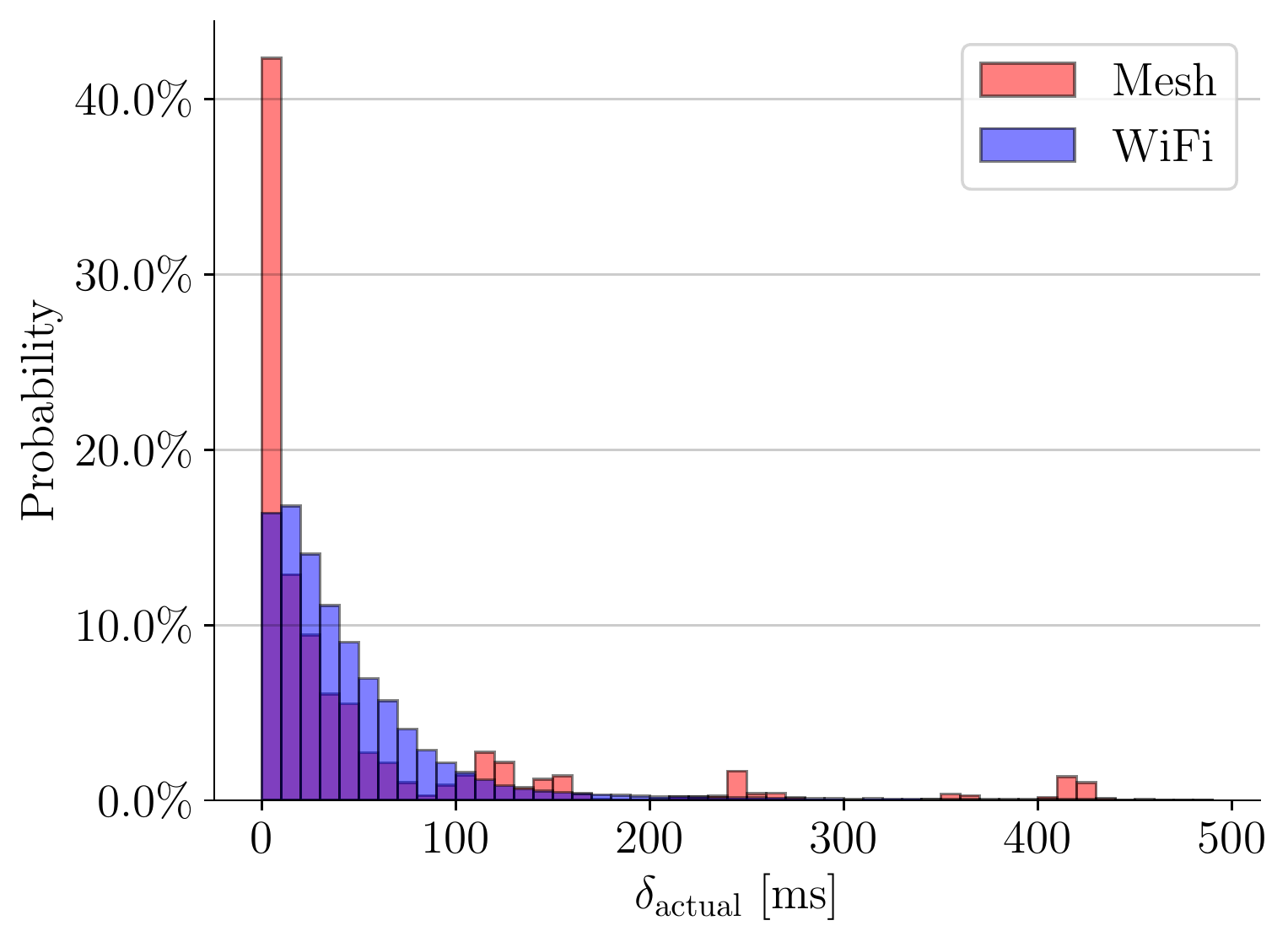}
    \label{fig:comm_delay_on_mesh}
    }
    \quad
    \subfloat[\scriptsize Centralized]{
        \scalebox{0.6}{
        \begin{tikzpicture}
            \node[inner sep=0pt] (router) at (0.8,2)
                {\includegraphics[width=.15\columnwidth]{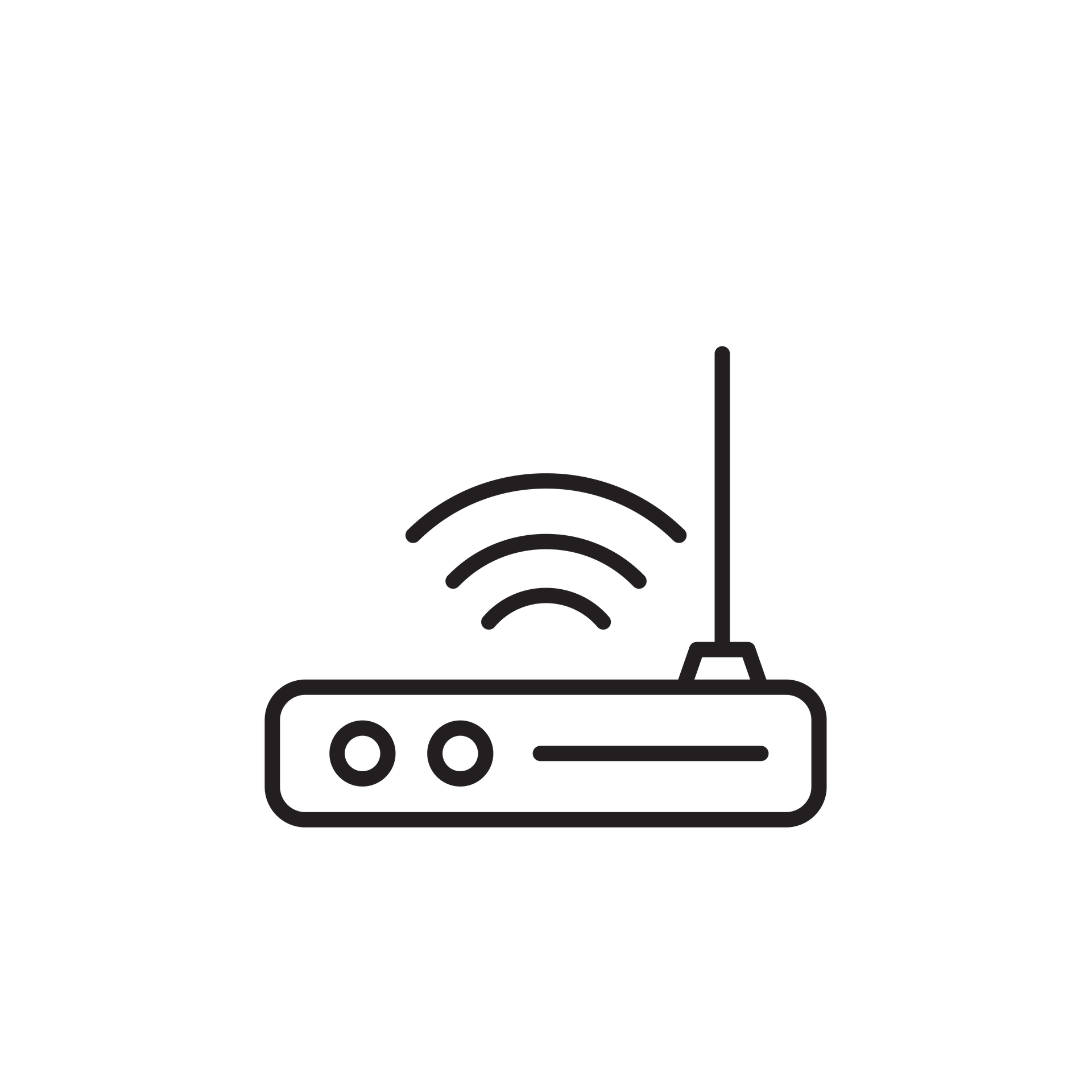}};
            \node[inner sep=0pt] (UAV1) at (0,0)
                {\includegraphics[width=.06\columnwidth]{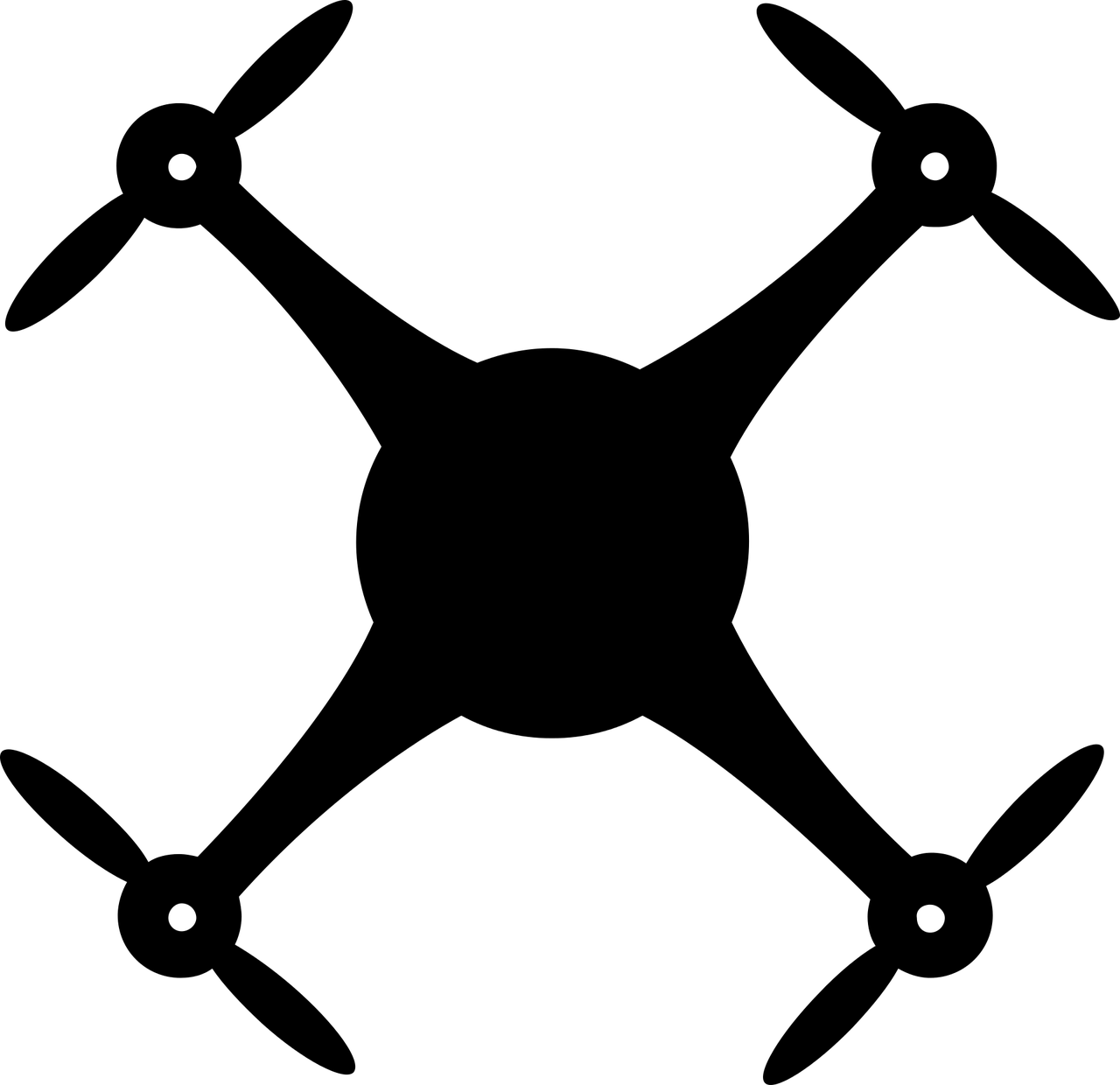}};
            \node[inner sep=0pt] (UAV2) at (0.7,0)
                {\includegraphics[width=.06\columnwidth]{figures/UAV.png}};
            \node[inner sep=0pt] (UAV3) at (1.8,0)
                {\includegraphics[width=.06\columnwidth]{figures/UAV.png}};
            \path (UAV2) -- node[auto=false]{\ldots} (UAV3); 
            \draw[<->] (UAV1) -- (router);
            \draw[<->] (UAV2) -- (router);
            \draw[<->] (UAV3) -- (router);
        \end{tikzpicture}
        }
    }
    \quad
    \subfloat[\scriptsize Decentralized]{
    \scalebox{0.6}{
    \begin{tikzpicture}
        \node[inner sep=0pt] (UAV4) at (3,2)
            {\includegraphics[width=.06\columnwidth]{figures/UAV.png}};
        \node[inner sep=0pt] (UAV5) at (3.2,0.5)
            {\includegraphics[width=.06\columnwidth]{figures/UAV.png}};
        \node[inner sep=0pt] (UAV6) at (4,1.8)
            {\includegraphics[width=.06\columnwidth]{figures/UAV.png}};
        \node[inner sep=0pt] (UAV7) at (4.5,0)
            {\includegraphics[width=.06\columnwidth]{figures/UAV.png}};
        \path (5,1) -- node[auto=false]{\ldots} (5.5,1); 
        \draw[<->] (UAV4) -- (UAV5);
        \draw[<->] (UAV4) -- (UAV6);
        \draw[<->] (UAV4) -- (UAV7);
        \draw[<->] (UAV5) -- (UAV6);
        \draw[<->] (UAV5) -- (UAV7);
        \draw[<->] (UAV6) -- (UAV7);
        \draw[<-] (UAV4) -- (5,1);
        \draw[<-] (UAV5) -- (5,1);
        \draw[<-] (UAV6) -- (5,1);
        \draw[<-] (UAV7) -- (5,1);
    \end{tikzpicture}
    }
    }
    \caption{Communication network architectures and distribution of \delayActual{}: (a) distribution of \delayActual{}. (b) Agents must communicate via a centralized router. (c) Agents can communicate directly.}
    \label{fig:network}
    \vspace{-1em}
\end{figure}

We performed \textbf{22 hardware experiments} to demonstrate RMADER's robustness to communication delays as well as MADER's shortcomings: 5 flights for RMADER over a centralized network, 5 flights for MADER over a centralized network, 5 flights for RMADER over a decentralized (mesh) network, and 7 flights for RMADER on a decentralized (mesh) network with dynamic obstacles.
All the planning and control run onboard, and the state estimation is obtained by fusing IMU measurements with a motion capture system.
Each UAV has \texttt{Intel} NUC10 with 64GB RAM onboard.
Experiments 1--16 used \qtyproduct{9.2 x 7.5 x 2.5}{\m} flight space, and Experiments 17-22 used \qtyproduct{18.4 x 7.5 x 2.5}{\m}.

\subsection{Centralized vs. Decentralized (Mesh) Network}

We tested RMADER on both centralized and decentralized communication networks.
In a centralized network, a central device such as a WiFi router handles all communication traffic.
Note that the RMADER algorithm still runs onboard each UAV in a decentralized manner, even if the communication is centralized.
On the other hand, in a decentralized (mesh) network, agents communicate directly with each other.
These network architectures are illustrated in Fig.~\ref{fig:network}.
Note that in our mesh network configuration, all agents can communicate with each other, forming a full graph, meaning each agent has direct communication capabilities with every other agent in the network.
Although a centralized network is common in robotics, it has a single point of failure and is not scalable.

Fig.~\ref{fig:comm_delay_on_mesh} shows actual communication delays recorded on Experiments 1--22.
Overall, the mesh network has smaller communication delays compared to a centralized network; however, the mesh network has a few instances of large delays.
These outliers in communication delay are likely due to the use of low-cost wireless hardware and standard ad-hoc protocols and could be mitigated by a more advanced mesh networking system.
Because of these mesh network communication delay outliers, the 50-ms delay check used in all hardware experiments caused $\delayActualMax{}\le\delayParameter{}$ to be violated approximately 25\% of the time.
In these cases, RMADER cannot guarantee collision safety.
However, in Experiments 1--22, collisions are still avoided, and as discussed in~\cite{kondo2022robust}, 75th percentile coverage of communication delays significantly decreases collisions.
%

\subsection{RMADER vs. MADER on Centralized Network}
 A total of 10 hardware experiments (5 flights for each) demonstrate RMADER's robustness to communication delays as well as MADER's shortcomings. Each flight test had 6 UAVs and lasted \SI{1}{\minute}. 

During the MADER hardware experiments, due to the effects of communication delays, \textbf{7} potential collisions were detected. RMADER, on the other hand, did not generate conflicts. 
The maximum velocities and travel distances achieved in the hardware experiments are shown in Table~\ref{tab:mader_hw_centralized} and \ref{tab:rmader_hw_centralized}. Corresponding to simulation results in Section~\ref{sec:sim}, RMADER trades off safety with performance.

\subsection{RMADER on Decentralized (Mesh) Network}
We also performed RMADER on a decentralized (mesh) network. 
Table~\ref{tab:rmader_hw_mesh} shows RMADER's performance on a mesh network, and Fig.~\ref{fig:6agent_mesh} illustrates 6 agents carrying out deconfliction.
The collision-safety boundary box around each UAV was set to \qtyproduct{0.8 x 0.8 x 1.0}{\m}, and the dynamic limits were set to \SI{2.0}{\m/\s}, \SI{3.0}{\m/\s^2}, and \SI{4.0}{\m/\s^3}, except for Experiment 5, where we relaxed the dynamic constraints to \SI{4.0}{\m/\s}, \SI{5.0}{\m/\s^2}, and \SI{6.0}{\m/\s^3}.
Note that these dynamic constraints are element-wise, e.g., a UAV can achieve speed up to $4\sqrt{3} \approx \SI{6.93}{m/s}$.
We observed a deadlock at the very end of Experiment 12.
Compared to flight space, the boundary box used in these experiments is relatively large, and because we use hard constraints in optimization, this could lead to deadlock.

\begin{table}
\vspace{1em}
\caption{\centering RMADER on Mesh Metwork}
\begin{centering}
\renewcommand{\arraystretch}{1.0}
\resizebox{0.8\columnwidth}{!}{
\begin{tabular}{ c c c c c c }
\toprule
 & \textbf{Exp. 11} & \textbf{Exp. 12} & \textbf{Exp. 13} & \textbf{Exp. 14} & \makecell{\textbf{Exp. 15} \\ (fast)}\tabularnewline
\midrule 
\textbf{Max vel. [m/s]} & 2.7 & 3.0 & 2.9 & 2.8 & 3.0\tabularnewline
\hline 
\textbf{Avg. travel distance [m]} & 46.9 & 48.0 & 49.4 & 66.6 & 60.3 \tabularnewline
\hline
\textbf{Stop time [s]} & 13.7 & 10.6 & 21.4 & 6.2 & 8.8 \tabularnewline
\bottomrule
\end{tabular}}
\par\end{centering}
\label{tab:rmader_hw_mesh}
\end{table}

\subsection{RMADER with Dynamic Obstacles on Mesh Network}
This section illustrates a total of 7 hardware experiments.
The 2 and 4-agent experiments lasted \SI{1}{\minute}, and 6 agents exchange their position once. The dynamic constraints are \SI{3.0}{m/s}, \SI{4.0}{m/s^2}, and \SI{5.0}{m/s^3}, but in Experiments 16, 19 and 22, we increased them to 
\SI{5.0}{m/s}, \SI{7.0}{m/s^2}, and \SI{10.0}{m/s^3}, and the dynamic obstacles follow pre-determined trefoil trajectories. 
Table~\ref{tab:rmader_with_obstacles} shows RMADER's performance with obstacles, and Figs~\ref{fig:2agent1obs}, ~\ref{fig:4agent2obs}, and \ref{fig:6agent2obs} illustrate 2, 4, and 6 agents with RMADER running onboard successfully carrying out position exchange while avoiding dynamic obstacles.    
Due to a hardware issue, one of the 6 agents in Experiments 20--22 was not able to publish its trajectory; however, because it was receiving trajectories from other agents and because of RMADER's decentralized deconfliction mechanism, we observed no collisions in these experiments. 
\begin{table}[t]
\vspace{1em}
\caption{\centering Hardware experiments with dynamic obstacles}
\label{tab:rmader_with_obstacles}
\begin{centering}
\renewcommand{\arraystretch}{1.2}
\resizebox{0.8\columnwidth}{!}{
\begin{tabular}{ c c | c c c | c c c }
\toprule
& \multicolumn{1}{c}{\makecell{\textbf{2 agents} \\ (1min)}} & \multicolumn{3}{c}{\makecell{\textbf{4 agents} \\ (1min) }} &  \multicolumn{3}{c}{\makecell{\textbf{6 agents} \\ (one time) }}\tabularnewline
 & \makecell{\textbf{Exp. 16} \\ (fast)} & \textbf{Exp. 17} & \textbf{Exp. 18} & \makecell{\textbf{Exp. 19} \\ (fast)} & \textbf{Exp. 20} & \textbf{Exp. 21} & \makecell{\textbf{Exp. 22} \\ (fast)}  \tabularnewline
\midrule
\textbf{Max vel. [m/s]} & 4.7 & 3.8 & 3.6 & 5.8 & 3.2 & 3.5 & 5.6 \tabularnewline
\hline 
\makecell{\textbf{Avg. travel} \\ \textbf{distance [m]}} & 106.7 & 114.9 & 109.9 & 114.4 & 24.6 & 23.7 & 22.2 \tabularnewline
\hline
\textbf{Stop time [s]} & 1.60 & 3.51 & 0.90 & 4.91 & 2.42 & 0.76 & 0.76 \tabularnewline
\bottomrule
\end{tabular}
}
\par\end{centering}
\vspace{-1em}
\end{table}

\begin{figure}[!htbp]
    \centering
    \subfloat[\scriptsize 6 agents on mesh network.\label{fig:6agent_mesh}]{
    \resizebox{0.5\columnwidth}{!}{
    \input{tikz-figures/6-wo-obsts}
    }
    }
    \subfloat[\scriptsize 2 agents with a dynamic obstacle.\label{fig:2agent1obs}]{
    \centering
    \resizebox{0.5\columnwidth}{!}{
    \begin{tikzpicture}
    \node (img) {\includegraphics[width=\columnwidth, height=0.22\textheight, clip, keepaspectratio]{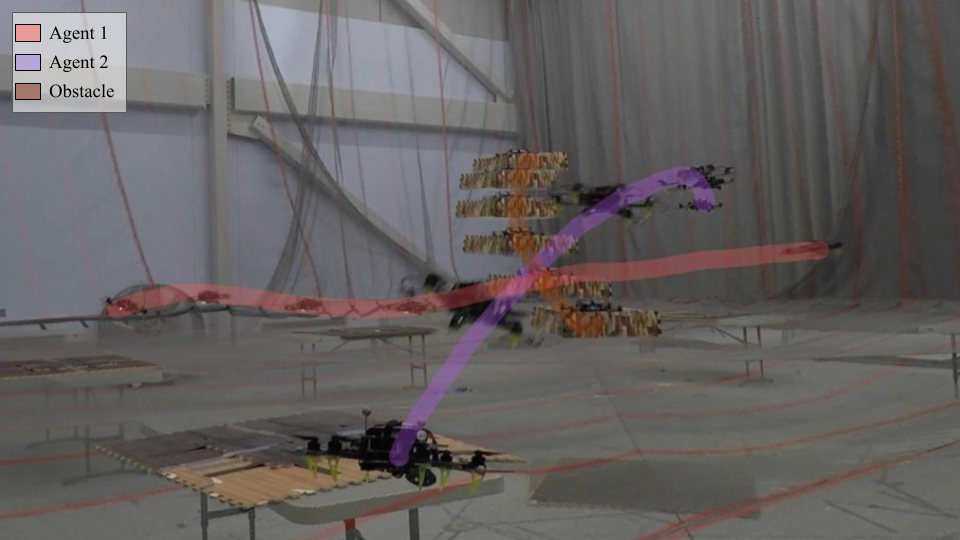}};
    \filldraw[color=black, fill=agent1_color] (3.1, 0.2) circle (2pt);
    \filldraw[color=black, fill=agent2_color] (-0.7, -1.7) circle (2pt);
    \filldraw[color=black, fill=obstacle_color] (0.4, 1.0) circle (2pt);
    \node [rectangle, draw, fill=agent1_color, inner sep=0.7mm] at (-3.3, -0.4) {};
    \node [rectangle, draw, fill=agent2_color, inner sep=0.7mm] at (2.0, 0.6) {};
    \node [rectangle, draw, fill=obstacle_color, inner sep=0.7mm] at (1.0, -0.5) {};
    \end{tikzpicture}
    }
    }
    \\
    \subfloat[\scriptsize 4 agents with 2 dynamic obstacles.\label{fig:4agent2obs}]{
    \resizebox{0.5\columnwidth}{!}{
    \begin{tikzpicture}
    \node (img) {\includegraphics[width=\columnwidth, height=0.22\textheight, keepaspectratio]{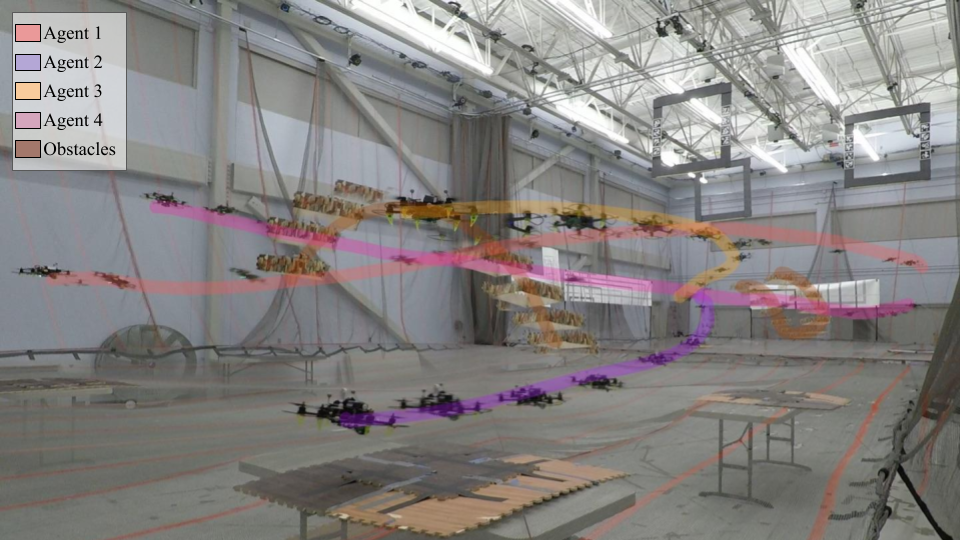}};
    \filldraw[color=black, fill=agent1_color] (-3.9, -0.1) circle (2pt);
    \filldraw[color=black, fill=agent2_color] (1.9, -0.2) circle (2pt);
    \filldraw[color=black, fill=agent3_color] (-0.7, 0.5) circle (2pt);
    \filldraw[color=black, fill=agent4_color] (3.8, -0.3) circle (2pt);
    \filldraw[color=black, fill=obstacle_color] (-1.65, 0.05) circle (2pt);
    \filldraw[color=black, fill=obstacle_color] (2.6, 0) circle (2pt);
    \node [rectangle, draw, fill=agent1_color, inner sep=0.7mm] at (3.9, 0) {};
    \node [rectangle, draw, fill=agent2_color, inner sep=0.7mm] at (-1.1, -1.4) {};
    \node [rectangle, draw, fill=agent3_color, inner sep=0.7mm] at (1.8, -0.2) {};
    \node [rectangle, draw, fill=agent4_color, inner sep=0.7mm] at (-2.8, 0.6) {};
    \node [rectangle, draw, fill=obstacle_color, inner sep=0.7mm] at (0.65, -0.7) {};
    \node [rectangle, draw, fill=obstacle_color, inner sep=0.7mm] at (2.35, -0.2) {};
    \end{tikzpicture}
    }
    }
    \subfloat[\scriptsize 6 agents with 2 dynamic obstacles.\label{fig:6agent2obs}]{
    \centering
    \resizebox{0.5\columnwidth}{!}{
    \begin{tikzpicture}
    \node (img) {\includegraphics[width=\columnwidth, height=0.22\textheight, keepaspectratio]{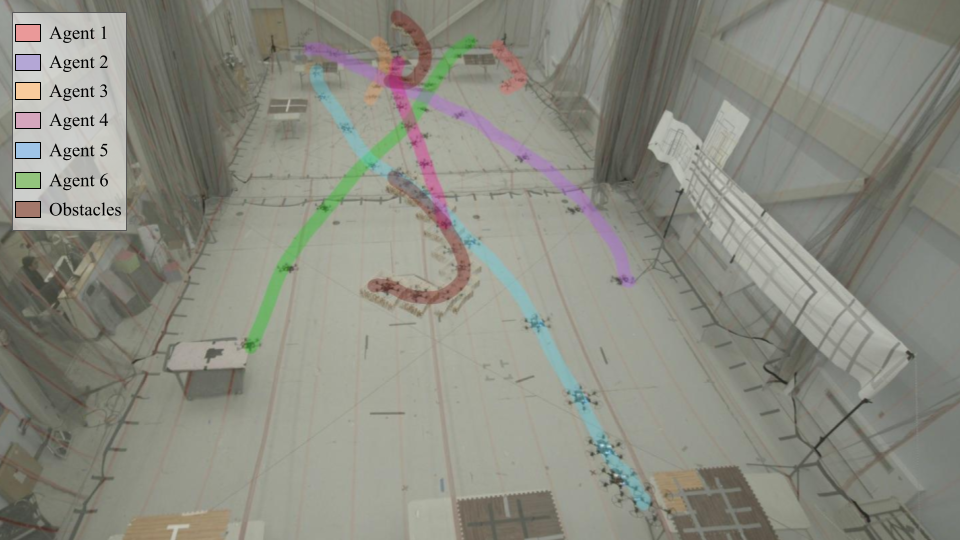}};
    \filldraw[color=black, fill=agent1_color] (0.2, 2.0) circle (2pt);
    \filldraw[color=black, fill=agent2_color] (1.3, -0.05) circle (2pt);
    \filldraw[color=black, fill=agent3_color] (-0.9, 2) circle (2pt);
    \filldraw[color=black, fill=agent4_color] (-0.3, 0.5) circle (2pt);
    \filldraw[color=black, fill=agent5_color] (-1.5, 1.9) circle (2pt);
    \filldraw[color=black, fill=agent6_color] (-2.0, -0.7) circle (2pt);
    \filldraw[color=black, fill=obstacle_color] (-0.9, -0.15) circle (2pt);
    \filldraw[color=black, fill=obstacle_color] (-0.7, 1.7) circle (2pt);
    \node [rectangle, draw, fill=agent1_color, inner sep=0.7mm] at (0.2, 1.6) {};
    \node [rectangle, draw, fill=agent2_color, inner sep=0.7mm] at (-1.6, 2) {};
    \node [rectangle, draw, fill=agent3_color, inner sep=0.7mm] at (-1, 1.4) {};
    \node [rectangle, draw, fill=agent4_color, inner sep=0.7mm] at (-0.8, 1.8) {};
    \node [rectangle, draw, fill=agent5_color, inner sep=0.7mm] at (1.6, -2.2) {};
    \node [rectangle, draw, fill=agent6_color, inner sep=0.7mm] at (0, 2.1) {};
    \node [rectangle, draw, fill=obstacle_color, inner sep=0.7mm] at (-0.7, 2.3) {};
    \node [rectangle, draw, fill=obstacle_color, inner sep=0.7mm] at (-0.8, 0.8) {};
    \end{tikzpicture}
    }
    }
    \caption{Agents move from $\bigcircle$ to \protect\tikz \protect\node [rectangle,draw] at (0,0) {};. Snapshots shown every \SI{500}{ms}.}
    \vspace{-1em}
\end{figure}

\section{CONCLUSIONS}
We introduced RMADER, a decentralized, asynchronous multiagent trajectory planner robust to communication delays. 
Our key contribution is ensuring safety through a delay check mechanism that maintains at least one collision-free trajectory. 
Simulations and hardware tests demonstrated RMADER's resilience to communication delays and effectiveness in dynamic environments. 



\section*{Appendix}\label{appendix}

Trajectory deconfliction cases 5-8 and 9-12 are presented in Figs.~\ref{fig:rmader_deconfliction-case-5-8} and \ref{fig:rmader_deconfliction-case-9-12}, respectively. 

\begin{figure}[H]
  \centering
  \begin{centering}      
  \resizebox{0.9\columnwidth}{!}{%
       \begin{tikzpicture}
       [
        greenbox/.style={shape=rectangle, fill=opt_color, draw=black},
        bluebox/.style={shape=rectangle, fill=check_color, draw=black},
         yellowbox/.style={shape=rectangle, fill=delaycheck_color, draw=black},
        ]
        
        \newcommand\Ay{2.5}
        \newcommand\Axo{1}
        \newcommand\Axc{3}
        \newcommand\Axr{4}
        \newcommand\Axe{5.5}
        
        \newcommand\By{0.7}
        \newcommand\Bxo{2.0}
        \newcommand\Bxc{3.4}
        \newcommand\Bxr{4.7}
        \newcommand\Bxe{6.2}
        
            \node[text=red] at (0.5,\Ay+0.2) {\scriptsize Agent A};
            \filldraw[fill=delaycheck_color, draw=black, opacity=0.2] (0,\Ay) rectangle (\Axo,\Ay-0.3);
            \filldraw[thick, fill=opt_color, draw=black] (\Axo,\Ay) rectangle (\Axc,\Ay-0.3);
            \filldraw[thick, fill=check_color, draw=black] (\Axc, \Ay) rectangle (\Axr, \Ay-0.3);
            \filldraw[thick, fill=delaycheck_color, draw=black] (\Axr, \Ay) rectangle (\Axe, \Ay-0.3);
            \filldraw[fill=opt_color, draw=black, opacity=0.2] (\Axe, \Ay) rectangle (\Axe+1.5, \Ay-0.3);
            \filldraw[fill=check_color, draw=black, opacity=0.2] (\Axe+1.5, \Ay) rectangle (\columnwidth, \Ay-0.3);
            \node[text=blue] at (0.5,\By+0.2) {\scriptsize Agent B};
            \filldraw[fill=check_color, draw=black, opacity=0.2] (0,\By) rectangle (\Bxo-1.5,\By-0.3);
            \filldraw[fill=delaycheck_color, draw=black, opacity=0.2] (\Bxo-1.5,\By) rectangle (\Bxo,\By-0.3);
            \filldraw[thick, fill=opt_color, draw=black] (\Bxo,\By) rectangle (\Bxc,\By-0.3);
            \filldraw[thick, fill=check_color, draw=black] (\Bxc, \By) rectangle (\Bxr, \By-0.3);
            \filldraw[thick, fill=delaycheck_color, draw=black] (\Bxr, \By) rectangle (\Bxe, \By-0.3);
            \filldraw[fill=opt_color, draw=black, opacity=0.2] (\Bxe, \By) rectangle (\columnwidth, \By-0.3);
        
        \draw[thick, densely dotted] (\Axr,-0.6) -- (\Axr,\Ay-0.3) node[] at (\Axr, -0.85) {\tiny t\textsubscript{traj\textsubscript{A\textsubscript{opt}}}};
        \draw[thick, densely dotted] (\Axe,-0.6) -- (\Axe,\Ay-0.3) node[] at (\Axe, -0.85) {\tiny t\textsubscript{traj\textsubscript{A\textsubscript{comm}}}};
            
        \draw[thick,->] (0,-0.6) -- (\columnwidth,-0.6) node[anchor=north east] {time};
        
        \draw[thick, ->, draw=red] (\Axr,\Ay-0.3) -- (\Axr,\Ay-1.2) node[midway,fill=white, fill opacity=0.8, text opacity=1, text=red] {\tiny \trajAOpt{}};
        \draw[thick, ->, draw=red] (\Axe,\Ay-0.3) -- (\Axe,\Ay-1.2) node[midway,fill=white, fill opacity=0.8, text opacity=1, text=red] {\tiny \trajAComm{}};
        \draw[thick, <-, draw=red] (\Bxc-0.2,\By) -- (\Bxc-0.2,\By+0.3)  node[anchor=south,text=black] {\tiny \makecell{case 5 \\ (not occur)}};
        \draw[thick, <-, draw=red] (\Bxc+0.85,\By) -- (\Bxc+0.85,\By+0.3) node[anchor=south,text=black] {\tiny case 6};
        \draw[thick, <-, draw=red] (\Bxr+0.2,\By) -- (\Bxr+0.2,\By+0.3) node[anchor=south,text=black] {\tiny case 7};
        \draw[thick, <-, draw=red] (\Bxe+0.4,\By) -- (\Bxe+0.4,\By+0.3) node[anchor=south,text=black] {\tiny \makecell{case 8 \\ (not occur)}};
        \draw[thick, <->, draw=black] (\Axe,-0.2) -- (\columnwidth,-0.2) node[midway, anchor=south, text=black] {\tiny \textbf{\textcolor{red}{\trajAOpt{}} will not arrive after \DCStepA{}}};
        
        \node[font=\bfseries,right] at (\Axo,\Ay-0.15) {\tiny O\textsubscript{A}};
        \node[font=\bfseries,right] at (\Axc,\Ay-0.15) {\tiny C\textsubscript{A}};
        \node[font=\bfseries,right] at (\Axr,\Ay-0.15) {\tiny DC\textsubscript{A}};

        \node[font=\bfseries,right] at (\Bxo,\By-0.15) {\tiny O\textsubscript{B}};
        \node[font=\bfseries,right] at (\Bxc,\By-0.15) {\tiny C\textsubscript{B}};
        \node[font=\bfseries,right] at (\Bxr,\By-0.15) {\tiny DC\textsubscript{B}};
        
        
        \node[color=gray] at (0.5,\Ay-0.15) {\scriptsize Prev. iter.};
        \node[color=gray] at (0.95\columnwidth,\Ay-0.15) {\scriptsize Next iter.};
        \node[color=gray] at (0.5,\By-0.15) {\scriptsize Prev. iter.};
        \node[color=gray] at (0.95\columnwidth,\By-0.15) {\scriptsize Next iter.};
        
    \end{tikzpicture}
    }
    \captionsetup{singlelinecheck=off}
    \caption[RMADER Deconfliction Case 5-8]{RMADER deconfliction Cases 5-8.}
    \label{fig:rmader_deconfliction-case-5-8}
    \end{centering}
    \vspace{-1em}
\end{figure}
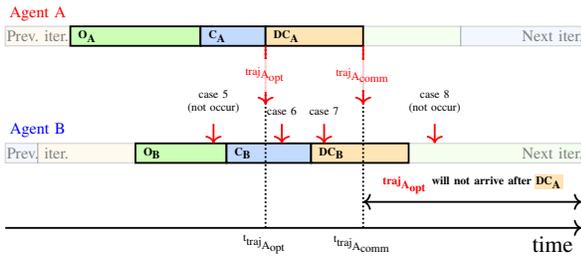

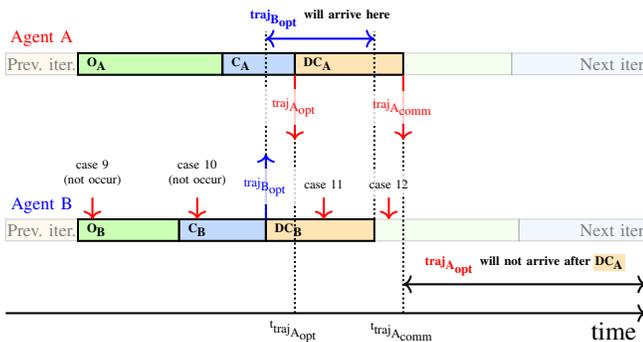
\begin{figure}[H]
  \centering
  \begin{centering}
  \resizebox{\columnwidth}{!}{%
       \begin{tikzpicture}
       [
        greenbox/.style={shape=rectangle, fill=opt_color, draw=black},
        bluebox/.style={shape=rectangle, fill=check_color, draw=black},
         yellowbox/.style={shape=rectangle, fill=delaycheck_color, draw=black},
        ]
        
        \newcommand\Ay{3}
        \newcommand\Axo{1}
        \newcommand\Axc{3}
        \newcommand\Axr{4}
        \newcommand\Axe{5.5}
        
        \newcommand\By{0.7}
        \newcommand\Bxo{1.0}
        \newcommand\Bxc{2.4}
        \newcommand\Bxr{3.6}
        \newcommand\Bxe{5.1}
        
            \node[text=red] at (0.5,\Ay+0.2) {\scriptsize Agent A};
            \filldraw[fill=delaycheck_color, draw=black, opacity=0.2] (0,\Ay) rectangle (\Axo,\Ay-0.3);
            \filldraw[thick, fill=opt_color, draw=black] (\Axo,\Ay) rectangle (\Axc,\Ay-0.3);
            \filldraw[thick, fill=check_color, draw=black] (\Axc, \Ay) rectangle (\Axr, \Ay-0.3);
            \filldraw[thick, fill=delaycheck_color, draw=black] (\Axr, \Ay) rectangle (\Axe, \Ay-0.3);
            \filldraw[fill=opt_color, draw=black, opacity=0.2] (\Axe, \Ay) rectangle (\Axe+1.5, \Ay-0.3);
            \filldraw[fill=check_color, draw=black, opacity=0.2] (\Axe+1.5, \Ay) rectangle (\columnwidth, \Ay-0.3);
            \node[text=blue] at (0.5,\By+0.2) {\scriptsize Agent B};
            \filldraw[fill=delaycheck_color, draw=black, opacity=0.2] (0,\By) rectangle (\Bxo,\By-0.3);
            \filldraw[thick, fill=opt_color, draw=black] (\Bxo,\By) rectangle (\Bxc,\By-0.3);
            \filldraw[thick, fill=check_color, draw=black] (\Bxc, \By) rectangle (\Bxr, \By-0.3);
            \filldraw[thick, fill=delaycheck_color, draw=black] (\Bxr, \By) rectangle (\Bxe, \By-0.3);
            \filldraw[fill=opt_color, draw=black, opacity=0.2] (\Bxe, \By) rectangle (\Bxe+2.0, \By-0.3);
            \filldraw[fill=check_color, draw=black, opacity=0.2] (\Bxe+2.0, \By) rectangle (\columnwidth, \By-0.3);
        
        \draw[thick, densely dotted] (\Axr,-0.6) -- (\Axr,\Ay-0.3) node[] at (\Axr, -0.85) {\tiny t\textsubscript{traj\textsubscript{A\textsubscript{opt}}}};
        \draw[thick, densely dotted] (\Axe,-0.6) -- (\Axe,\Ay-0.3) node[] at (\Axe, -0.85) {\tiny t\textsubscript{traj\textsubscript{A\textsubscript{comm}}}};
        \draw[thick, densely dotted] (\Bxr,\By) -- (\Bxr,\Ay+0.3);
        \draw[thick, densely dotted] (\Bxe,\By) -- (\Bxe,\Ay+0.3);
        \draw[thick, <->, draw=blue] (\Bxr,\Ay+0.2) -- (\Bxe,\Ay+0.2) node[midway, anchor=south, text=black] {\tiny \textbf{\textcolor{blue}{\trajBOpt{}} will arrive here}};
            
        \draw[thick,->] (0,-0.6) -- (\columnwidth,-0.6) node[anchor=north east] {time};
        
        \draw[thick, ->, draw=red] (\Axr,\Ay-0.3) -- (\Axr,\Ay-1.2) node[midway,fill=white, fill opacity=0.8, text opacity=1, text=red] {\tiny \trajAOpt{}};
        \draw[thick, ->, draw=red] (\Axe,\Ay-0.3) -- (\Axe,\Ay-1.2) node[midway,fill=white, fill opacity=0.8, text opacity=1, text=red] {\tiny \trajAComm{}};
        \draw[thick, ->, draw=blue] (\Bxr,\By) -- (\Bxr,\By+0.9) node[midway,fill=white, fill opacity=0.8, text opacity=1, text=blue] {\tiny \trajBOpt{}};
        \draw[thick, <-, draw=red] (\Bxo+0.2,\By) -- (\Bxo+0.2,\By+0.3)  node[anchor=south,text=black] {\tiny \makecell{case 9 \\ (not occur)}};
        \draw[thick, <-, draw=red] (\Bxc+0.25,\By) -- (\Bxc+0.25,\By+0.3) node[anchor=south,text=black] {\tiny \makecell{case 10 \\ (not occur)}};
        \draw[thick, <-, draw=red] (\Bxr+0.8,\By) -- (\Bxr+0.8,\By+0.3) node[anchor=south, fill=white, fill opacity=0.8, text opacity=1, text=black] {\tiny case 11};
        \draw[thick, <-, draw=red] (\Bxe+0.2,\By) -- (\Bxe+0.2,\By+0.3) node[anchor=south, fill=white, fill opacity=0.8, text opacity=1, text=black] {\tiny case 12};
        \draw[thick, <->, draw=black] (\Axe,-0.2) -- (\columnwidth,-0.2) node[midway, anchor=south, text=black] {\tiny \textbf{\textcolor{red}{\trajAOpt{}} will not arrive after \DCStepA{}}};
        
        \node[font=\bfseries,right] at (\Axo,\Ay-0.15) {\tiny O\textsubscript{A}};
        \node[font=\bfseries,right] at (\Axc,\Ay-0.15) {\tiny C\textsubscript{A}};
        \node[font=\bfseries,right] at (\Axr,\Ay-0.15) {\tiny DC\textsubscript{A}};

        \node[font=\bfseries,right] at (\Bxo,\By-0.15) {\tiny O\textsubscript{B}};
        \node[font=\bfseries,right] at (\Bxc,\By-0.15) {\tiny C\textsubscript{B}};
        \node[font=\bfseries,right] at (\Bxr,\By-0.15) {\tiny DC\textsubscript{B}};
        
        
        \node[color=gray] at (0.5,\Ay-0.15) {\scriptsize Prev. iter.};
        \node[color=gray] at (0.95\columnwidth,\Ay-0.15) {\scriptsize Next iter.};
        \node[color=gray] at (0.5,\By-0.15) {\scriptsize Prev. iter.};
        \node[color=gray] at (0.95\columnwidth,\By-0.15) {\scriptsize Next iter.};
        
    \end{tikzpicture}
    }
    \captionsetup{singlelinecheck=off}
    \caption[RMADER Deconfliction Case 9-12.]{RMADER deconfliction Cases 9-12.}
    \label{fig:rmader_deconfliction-case-9-12}
    \end{centering}
    \vspace{-1em}
\end{figure}

\balance

\bibliographystyle{IEEEtran}


\bibliography{bib.bib} 

\end{document}